\documentclass{article}
\usepackage{microtype}
\usepackage{graphicx}
\usepackage{subcaption}
\usepackage{booktabs}
\usepackage{afterpage}
\usepackage{float}
\usepackage{placeins}
\usepackage{hyperref}
\usepackage{pdfpages}
\usepackage{tcolorbox}
\tcbuselibrary{skins}
\usepackage{enumitem}
\usepackage{microtype}
\usepackage{caption}

\tcbset{
  semanticeg/.style={
    enhanced,
    breakable,
    colback=white,
    colframe=black!65,
    boxrule=0.6pt,
    arc=1.5mm,
    outer arc=1.5mm,
    left=3pt,
    right=3pt,
    top=3pt,
    bottom=3pt,
    fontupper=\footnotesize\ttfamily, 
    before skip=8pt,
    after skip=8pt
  }
}

\newcommand{\OneWideFigure}[3]{
  \begin{figure}[H] 
    \centering
    \includegraphics[width=\linewidth,height=0.8\textheight,keepaspectratio]{#1}%
    \caption{#2}%
    \label{#3}%
  \end{figure}
}

\usepackage{algorithm}

\usepackage[preprint]{icml2026}

\usepackage{amsmath}
\usepackage{amssymb}
\usepackage{mathtools}
\usepackage{amsthm}

\usepackage[capitalize,noabbrev]{cleveref}

\theoremstyle{plain}
\newtheorem{theorem}{Theorem}[section]
\newtheorem{proposition}[theorem]{Proposition}
\newtheorem{lemma}[theorem]{Lemma}
\newtheorem{corollary}[theorem]{Corollary}
\theoremstyle{definition}
\newtheorem{definition}[theorem]{Definition}
\newtheorem{assumption}[theorem]{Assumption}
\theoremstyle{remark}
\newtheorem{remark}[theorem]{Remark}

\usepackage[textsize=tiny]{todonotes}

\icmltitlerunning{HALLUCINATION BASINS}

\begin{document}

\twocolumn[
  \icmltitle{Hallucination Basins: A Dynamic Framework for Understanding and Controlling LLM Hallucinations}

  \icmlsetsymbol{equal}{*}

  \begin{icmlauthorlist}
    \icmlauthor{Kalyan Cherukuri}{imsa}
    \icmlauthor{Lav R.\ Varshney}{x}
  \end{icmlauthorlist}

  \icmlaffiliation{imsa}{Illinois Mathematics and Science Academy, Illinois, USA}
  \icmlaffiliation{x}{AI Innovation Institute, Stony Brook University, New York, USA}

  \icmlcorrespondingauthor{Lav R.\ Varshney}{lav.varshney@stonybrook.edu}
  
  \vskip 0.3in
]

\printAffiliationsAndNotice{}  

\begin{abstract}
Large language models (LLMs) hallucinate: they produce fluent outputs that are factually incorrect. We present a geometric dynamical systems framework in which hallucinations arise from task-dependent basin structure in latent space. Using autoregressive hidden-state trajectories across multiple open-source models and benchmarks, we find that separability is strongly task-dependent rather than universal: factoid settings can show clearer basin separation, whereas summarization and misconception-heavy settings are typically less stable and often overlap. We formalize this behavior with task-complexity and multi-basin theorems, characterize basin emergence in $L$-layer transformers, and show that geometry-aware steering can reduce hallucination probability without retraining. 

\end{abstract}

\section{Introduction}
Recent advances in machine learning have demonstrated large language models (LLMs) with strong linguistic and reasoning capabilities. However, hallucinations, which are fluent outputs by the model that are either semantically false or factually false are a critical challenge in deploying LLMs for sensitive, real-world  tasks. Prior research has largely treated hallucination detection  as an output-level problem, using entropy or uncertainty measures to flag unreliable answers. These methods do not inherently explain \textit{why} hallucinations occur, and require labeled data or external knowledge for explanation. Without fully understanding the mechanism, high-stakes domains where LLMs can be employed remain inherently risky. For example in medical diagnosis, legal reasoning, and scientific research, factual accuracy is needed.

\textbf{Why do factoid and generation tasks exhibit opposing geometric structure?} Existing hallucination detection methods 
consider universal mechanisms: semantic entropy, linear probing, or methods to quantify uncertainty. The literature has no clear foundational understanding as to what is happening in LLMs to induce this hallucination phenomena. 

\textbf{Our insight:} Hallucinations arise from a \textit{task-dependent geometric collapse}, determined by the cardinality of the set of valid answers (i.e., factoid tasks with single answers induce a point attractor; generation tasks with many valid outputs form high-dimensional manifolds; misconception tasks create indistinguishable basins). This work aims to explain why this happens and an approach.
\subsection*{Contributions}
\begin{enumerate}
    \item \textbf{Hallucination Basins.} We introduce and formalize the behavior of hallucinations as a dynamical systems phenomenon and define reference states, basins, and radial contraction properties to explain how outputs collapse to context-insensitive points. \vspace{-2mm}

    \item \textbf{Single- and Multi-Basins.} We establish that basin geometry is task-dependent: factoids exhibit single-basin collapse, whereas misconception tasks form multiple basins, where competing answers create distinct, high-confidence attractors. \vspace{-2mm}

    \item \textbf{Causal Intervention.} Pushing factual representation vectors towards hallucination basins increases the probability of hallucination. This result directly indicates the presence of an attractor-like \textit{basin}. 
    \vspace{-2mm}

    \item \textbf{Adaptive Geometry-Aware Steering.} We develop a lightweight geometric steering method that applies latent shifts based on basin proximity and empirically shows a reduction in generating hallucinations without the need for retraining.

\end{enumerate}

\section{Related Work}
Recent research on LLM hallucinations splits into: (1) output-level uncertainty methods, (2) representation-level probes and detectors, and (3) intervention/steering approaches to change the model input/prompt or latent states. Our paper shifts this narrative by reframing hallucinations as a dynamical systems phenomenon:  
attractor-like hallucination basins in layerwise latent spaces.

\textbf{Uncertainty and Output-Based Detection.}
Many papers treat hallucination as an output-based uncertainty phenomenon. Zero-resource and black-box approaches, e.g.\ SelfCheckGPT~\cite{manakul2023selfcheckgpt}, rely on sampling inconsistency. More recent studies formalize the limits of hallucination detection, showing automated detection processes are fundamentally constrained \cite{karbasi2025impossibility}. Surveys on surface-level uncertainty and retrieval errors include~\citet{alansari2025large, huang2025survey}. While these methods are effective in some settings, they do not fully explain \textit{why} hallucinations arise, nor why detection performance collapses on generation- or misconception-heavy tasks.

\textbf{Probe and Representation-Based Classifiers.}
Several works move beyond outputs and internal representations. INSIDE \cite{chen2024inside} shows that hidden states remain a predictive signal for hallucination detection, whereas LLM-Check \cite{sriramanan2024llm} systematically evaluates for probing-based detectors. Sharpness-dependent metrics argue that factual generations inherently correspond to lower entropy and thus more concentrated internal activations \cite{pmlr-v235-chen24av}. Mechanistic interpretability approaches such as ReDEEP \cite{sun2024redeep} and InterpDetect \cite{tan2025interpdetect} analyze the latent features in retrieval-augmented generation (RAG). However, these methods are largely empirical observations; they identify correlations with hallucinations but fail to provide a geometric or a dynamical setup explaining why these signals must exist.

\textbf{Steering.}
Another line of work attempts to mitigate hallucinations. Multi-model contrastive decoding and dynamic detection have been proposed as decoding-time safeguards \cite{zhualleviating}. Latent-space steering methods work to modify the LLM's internal representations to reduce hallucinations \cite{sahoo2024addressing}. Memory space retracing in multimodal models suggests that revising the internal memory states can improve factual accuracy \cite{zou2025look}. ACT (Adaptive Activation Steering) 
 works to show that using a diverse set of `truthfulness' steering vectors applied to shift the LLM's activations towards truthful answers \cite{wang2025adaptive}. Our work contributes to this literature with a method that uses the geometric structure (basin attractors) to create a steering mechanism.

\textbf{Associative Memory, Attractors, and Bi- Multi-stability.}
Our work is strongly motivated and justified by classical and modern theories of associative memory \cite{doi:10.1073/pnas.79.8.2554, inazawa2025associative}. Hopfield networks and their higher-order, rotor, and multistable variants demonstrate how neural systems naturally develop basins of attraction that retrieve stored patterns \cite{CHEN2025128893,LI2025116657,essex2025memorisation}. Biological and physical systems show similar multistability and competing basins \cite{10.1098/rstb.2019.0765}. Recent biologically-grounded associative memory models further emphasize retrieval via basin convergence \cite{kafraj2025a}. Recent studies have started to connect hallucinations to internal references and memory retrieval \cite{sun2025why}.  
In parallel, modern architectures such as the Associative Transformer explicitly incorporate this biologically-inspired idea for associative recall mechanisms \cite{sun2025associative}. These works provide theoretical basis for viewing LLM behavior via attractor dynamics; recall the direct relationship between Hopfield networks and Transformer architectures \cite{Ramsauer_ea2021}.

\section{Preliminaries}

Table~\ref{tab:notation} has key notation that will be used throughout.

\begin{table}
\centering
\caption{Notation summary}
\label{tab:notation}
\small
\begin{tabular}{@{}ll@{}}
\toprule
\textbf{Symbol} & \textbf{Definition} \\
\midrule
$h^{(\ell)}$ & Hidden state at layer $\ell \in \{0, \ldots, L\}$ \\
$d$ & Dimension of hidden states ($h^{(\ell)} \in \mathbb{R}^d$) \\
$\mu^{(\ell)}$ & Reference/centroid state at layer $\ell$ (basin center) \\
$\mathcal{B}^{(\ell)}(r)$ & Basin of attraction: $\{h : \|h - \mu^{(\ell)}\|_2 \le r\}$ \\
$J_\ell$ & Jacobian $\partial f_\ell / \partial h$ at layer $\ell$ \\
$\rho(\cdot)$ & Spectral radius (largest eigenvalue magnitude) \\
$P_V$ & Projection onto subspace $V$ (mean-zero subspace) \\
$V$ & Mean-zero subspace: $\{\delta h : \mathbf{1}^\top \delta h = 0\}$ \\
$\alpha_j^{(\ell)}$ & Attention weight for token $j$ at layer $\ell$ \\
$H_{\text{attn}}$ & Attention entropy: $-\sum_j \alpha_j \log \alpha_j$ \\
$d_{\text{basin}}^{(\ell)}$ & Distance to basin center: $\|h^{(\ell)} - \mu^{(\ell)}\|_2$ \\
$\rho_{\text{Fisher}}^{(\ell)}$ & Fisher discriminant ratio (between/within-class) \\
$\gamma_{\text{LN}}$ & LayerNorm centering coefficient ($< 1$) \\
$\gamma_{\text{FFN}}$ & FFN contraction coefficient ($< 1$) \\
$\delta^2$ & The squared value for the Mahalanobis distance
\end{tabular}
\end{table}

\subsection{Language Models and Notation}
We consider a standard decoder-only transformer LLM. Let $\mathcal{V}$ be the token vocabulary and $X = (x_1,\dots,x_n)$ be a sequence of tokens (the context or input prompt). The model computes hidden states $h_0, h_1, \dots, h_n \in \mathbb{R}^d$, where $h_0$ is a learned start token embedding and each $h_i$ following is obtained by applying some $L$ transformer layers. In formal terms, each layer $l = 1, \dots , L$ applies a self-attention and MLP transformation to produce $h_i^l$ from $h_i^{l-1}$. The final layer output $h_n^L$ is fed to a linear and softmax operator to define the distribution of the next token:
\[
P(y |X) = \mathrm{Softmax}(W h_n^L + b)\,,\quad W\in\mathbb{R}^{|\mathcal V|\times d}.
\]
At step $t$ the model's conditional distribution is given by $P(\cdot\mid h_t^L)$ or when the context is clear by $P(\cdot|h)$.

\subsection{Embeddings and Representation Space}

Consider the latent representation space as $\mathbb{R}^d$ with the Euclidean metric. In this space, each token has a corresponding embedding, and the hidden states also live in this same space, $h_i^l\in\mathbb{R}^d$ at each layer $l$. We consider the final-layer space $H = \mathbb{R}^d$ to respond to the norm $|\cdot|_2$. 
Other divergences may be considered, but the most natural is the Euclidean distance considering the Transformer's linear layers.

\subsection{Layer-wise Latent Activation Trajectories}

Given a completed sequence (a context plus generated tokens), the latent trajectory of the $i$th token is the sequence of $h_i^0, h_i^1, \dots, h_i^L$ across layers (with $h_i^0$ the embedding of token $i$ and $h_i^L$ the final hidden state which is utilized to predict the token $i+1$). An equivalent way to view this is that the entire generation is a trajectory of the final hidden state after each token. The key point is that each new token's prediction is determined by its hidden trajectory. For the purpose of simplicity, we often analyze a single token's trajectory through layers, since intervening context is represented within its input $h_i^0$ and attention.

\subsection{Hallucinations}

We take a distributional view on hallucination. Intuitively, a generated token is considered to be a hallucination if it is fluent but not grounded within the context. Formally, suppose the model output $y$ has a high probability under the model but is not the true grounded completion. One way to capture this is with the conditional distribution of the model.

\begin{definition}[Answer Cardinality]
\label{def:cardinality}
For a task $T$, let $\mathcal{A}$ be the set of valid completions; \textit{answer cardinality} is $|\mathcal{A}|$.
\end{definition}
\begin{itemize}
    \item \textbf{Factoid Tasks} (i.e., QA, fact verification): $|\mathcal{A}| = 1$. Indicating that a unique correct answer exists. \vspace{-1,5mm}
    \item \textbf{Generation Tasks} (i.e., summarization):$|\mathcal{A}| \to \infty$ (there exist infinitely many valid outputs). \vspace{-1,5mm}
    \item \textbf{Misconception Tasks} (i.e., multiple plausible but incorrect answers): $|\mathcal{A}| \approx 2-5$ (any set with finite and a countable number of solutions). 
\end{itemize}

\section{Problem Formulation}

We assume access to an LLM, $f_\theta$, alongside its layer-wise hidden states, but no ground truth oracle at the point of inference. At test time, the model is given a prompt $X$ and generates tokens sequentially, $y_1, y_2, \dots$. We want to understand \textit{when and why} $f_\theta$ can generate a hallucination. To do this, we can monitor the hidden states $h_n^l$ at each layer $\ell \in \{1,\ldots, L\}$ during generation. We aim to determine the hallucination risk solely from these internal signals, without using external data. Thus the presented framework is self-contained: everything is rooted within the model's latent geometry and its conditional distribution. 

Existing methods do not wholly capture our phenomenon. Uncertainty-based detectors \cite{Farquhar2024SemanticEntropy} compute entropy or mutual information on $P(y|X)$, but only look at the surface distribution and often require calibration or multiple samples. Probe-based methods \cite{park2025steer}---like training a small `hallucination' vs.\ `truth' classifier on hidden states---rely on labeled examples or other heuristic measurements. They can flag hallucinations ex post, but do not thoroughly explain the underlying cause. Critically, none of these approaches link hallucination probability to geometric properties of hidden trajectories. Such methods cannot predict how changes in representation (layer $l$ to $l+1$) affect hallucination risk. We seek an explicit connection, asking how distances, volumes, and curvature in the latent space bound or determine the likelihood that the model ``runs away'' into a hallucination mode. 

\section{Hallucination Basins}

\subsection{Reference State Construction}
To define basins independent of specific tasks, we construct reference states from contexts that are not informative. Let $\mathcal{C}$ denote a distribution over the contexts that are semantically uninformative or weakly informative (e.g., empty strings, single token outputs, short generic phrases like ``The'', ``Hello'', etc.). We sample $|\mathcal{C}| = 1000$ uninformative contexts uniformly from such a distribution. For each layer $\ell \in \{1,\ldots,L\}$, define the subsequent \textit{reference state}:
\[
\mu^{(\ell)} = \mathbb{E}_{x\sim \mathcal{C}} \left[ h^{(\ell)}(x) \right]  \approx \frac{1}{|\mathcal{C}|}\sum_{x\in\mathcal{C}} h^{(\ell)}(x)
\mbox{.}
\]
In practice, the empirical mean over single-token prompts from a vocabulary subset ensures computational feasibility across models.

\begin{proposition}[Reference states as fixed points]
\label{prop:fixed-points}
If attention over some uninformative contexts as described above, $x \sim \mathcal{C}$, concentrates uniformly then $\mathbb{E}_{x\sim\mathcal{C}}[\text{Attn}^{(\ell)}(h^{(\ell-1)}(x))] \approx 0$ and thus:
\[
\|f_\ell(\mu^{(\ell)}) - \mu^{(\ell+1)}\|_2 = O(\sigma_{\mathcal{C}}),
\]
where $\sigma_{\mathcal{C}}$ is the variance of $h^{(\ell)}(x)$ over $x \sim \mathcal{C}$. Additionally, if the Jacobian, $J_\ell(\mu^{(\ell)})$ has spectral radius $\rho(J_\ell(\mu^{(\ell)})) < 1$, then $\mu^{(\ell)}$ is an approximate attracting fixed point.
\end{proposition}
\begin{proof}
For $x \in \mathcal{C}$, weak query-key alignment yields $\alpha_{ij} \approx 1/n$, so $\text{Attn}(h) \approx \frac{1}{n}\sum_j v_j \to 0$ in expectation over centered embeddings. The residual update $h^{(\ell)} = h^{(\ell-1)} + \text{Attn}(\cdot) + \text{FFN}(\cdot)$ gives $\mathbb{E}[h^{(\ell)}] = \mu^{(\ell-1)} + O(\sigma_{\mathcal{C}}^2)$, so $\|f_\ell(\mu^{(\ell)}) - \mu^{(\ell+1)}\|_2 = O(\sigma_{\mathcal{C}})$. Spectral radius $\rho < 1$ ensures contraction.
\end{proof}

\begin{definition}[Reference region]
For a fixed layer $\ell$ and radius $r > 0$, define the \emph{reference region}
\[
\mathcal{B}_\ell(r)
\;:=\;
\left\{
h \in \mathbb{R}^d
\;\middle|\;
\left\| h - \mu^{(\ell)} \right\|_2 \le r
\right\}.
\]
\end{definition}

This reference region captures hidden states close to the model's default internal representation at layer $\ell$. Intuitively, such states encode weak dependence on the specific input and are dominated by architectural priors or priors induced via training. 

\begin{definition}[Hallucination basin]
For a layer, $\ell$ and radius $r>0$, the \textit{hallucination basin} is the ball:
\[
\mathcal{B}^{(\ell)}(r) := \left\{ h \in \mathbb{R}^d \mid \| h - \mu^{(\ell)} \|_2 \le r \right\}.
\]
With the two properties:
\begin{enumerate}
    \item \textbf{Attraction:} Trajectories that enter $\mathcal{B}^{(\ell)}(r)$ remain trapped in its subsequent layers
    \item \textbf{Insensitivity to Inputs:} Hidden states in $\mathcal{B}^{(\ell)}(r)$ produce identical output distributions regardless of the input context.
\end{enumerate}
\end{definition}
The radius of the basin, $r$, controls the range, where a larger $r$, increases the trapping phenomena probability but may include states that are grounded in the context. The stability of this state is given alternatively by Theorem~\ref{thm:traj_trap}.

\subsection{Basin Dynamics and Trajectory Trapping}

The mechanism behind the underlying hallucination events are that once trajectories enter a hallucination basin, the subsequent layers of the model contract representations back to the last reference state. This motivates the idea that recovery of context-specific (accurate) information is prevented.
\begin{definition}[Radial distance]
The layerwise radial distance is
$
r^{(\ell)}(x) = \big| h^{(\ell)}(x) - \mu^{(\ell)} \big|_2
$.
\end{definition}
This scalar process tracks how strongly the representation at each layer deviates from the reference geometry.

\begin{definition}[Radial contraction]
\label{def:rad_contraction}
A layer $\ell$ is said to be radially contractive on a set $S \subset \mathbb{R}^d$ if there exists an $\alpha_\ell < 1$ such that
\[
\| f_\ell(h) - \mu^{(\ell+1)} \|_2 \le \alpha_\ell \| h - \mu^{(\ell)} \|_2 \quad \forall h \in S.
\]
This contraction property is local and defined geometrically via analysis of the Jacobian near $\mu^{(\ell)}$.
\end{definition}
\begin{definition}[Subspace radial contraction]
Let $f_\ell:\mathbb{R}^d\to\mathbb{R}^d$ denote the layer-$\ell$ map. For a subspace $V\subseteq\mathbb{R}^d$ and set $S\subset\mathbb{R}^d$, we say $f_\ell$ is \emph{subspace radially contractive} on $(V,S)$ with constant $\alpha_\ell\in(0,1)$ if
\[
\|P_V\big(f_\ell(h)-\mu^{(\ell+1)}\big)\|_2 \le \alpha_\ell \,\|P_V(h-\mu^{(\ell)})\|_2
\qquad\forall h\in S,
\]
where $P_V$ is the orthogonal projection onto $V$.
\end{definition}
\begin{proposition}[Manifold attractor]
\label{prop:manifold_attractor}
Fix a layer, $\ell$, and suppose there is the existence of a smooth $k$-dimensional manifold $\mathcal{M}\subset\mathbb{R}^d$, of valid semantic states passing through $\mu^{(\ell)}$. Denote the tangential space $T_{\mu}\mathcal{M}\subset\mathbb{R}^d$ and the normal space $N_{\mu}\mathcal{M}=T_{\mu}\mathcal{M}^\perp$. Let $J_\ell(\mu)$ be the Jacobian of $f_\ell$ at $\mu^{(\ell)}$. Denote the orthogonal projections onto $T_{\mu}\mathcal{M}$ and $N_{\mu}\mathcal{M}$ by $P_T$ and $P_N$ respectively. Assume the following hold at a reference state, $\mu^{(\ell)}$: 
\begin{enumerate}
    \item \(\displaystyle \|P_N\,J_\ell(\mu)\,P_N\|_{2} \le \alpha_N < 1\)
    \item \(\displaystyle \|P_T\,J_\ell(\mu)\,P_T\|_{2} \le 1+\varepsilon_T\) for some small \(\varepsilon_T\ge 0\), and there exists at least one unit vector \(v\in T_{\mu}\mathcal{M}\) for which \(\|P_T J_\ell(\mu) v\|_2 \approx 1\)
\end{enumerate}

Then for initial perturbations $\delta_0$ sufficiently small, the iterated perturbation $\delta_{t+1}\approx J_\ell(\mu)\,\delta_t$ satisfies

\[
\begin{aligned}
\|P_N \delta_t\|_2 &\le C\,\alpha_N^t \|\delta_0\|_2, \\
\|P_T \delta_t\|_2 &= O\!\big((1+\varepsilon_T)^t\big)\,\|\delta_0\|_2.
\end{aligned}
\]

up to a constant $C>0$ which is independent to $t$. The subsequent result is that that perturbations orthogonal to $\mathcal{M}$ have an exponential decay, while perturbations tangential to $\mathcal{M}$ persists without reaching a contraction. If $\dim T_\mu\mathcal{M}=0$, the reference state $\mu^{(\ell)}$ is a locally attracting
fixed point. If $\dim T_\mu\mathcal{M}>0$ and $\varepsilon_T$ is small, trajectories are
attracted to a neighborhood of $\mathcal{M}$ and drift along it, creating a manifold attractor.

\end{proposition}

\begin{proof}
Linearizing the layer map at $\mu^{(\ell)}$ gives
\[
\delta_{t+1} = J_\ell(\mu)\,\delta_t.
\]
Decompose $\delta_t$ into orthogonal components
\[
\begin{aligned}
\delta_t &= \tau_t + \nu_t, \\
\tau_t &:= P_T \delta_t \in T_\mu\mathcal{M}, \\
\nu_t &:= P_N \delta_t \in N_\mu\mathcal{M}.
\end{aligned}
\]

Applying $J_\ell(\mu)$ and projecting yields
\[
\nu_{t+1} = P_N J_\ell(\mu) P_N \nu_t + P_N J_\ell(\mu) P_T \tau_t,
\]
\[
\tau_{t+1} = P_T J_\ell(\mu) P_T \tau_t + P_T J_\ell(\mu) P_N \nu_t.
\]

For sufficiently small $\|\delta_0\|_2$, the cross terms
$P_N J_\ell(\mu) P_T$ and $P_T J_\ell(\mu) P_N$ contribute only higher-order effects,
which can be absorbed into constants. Using the operator norm bounds,
\[
\begin{aligned}
\|\nu_{t+1}\|_2 \le \alpha_N \|\nu_t\|_2 + O(\|\tau_t\|_2),\\
\|\tau_{t+1}\|_2 \le (1+\varepsilon_T)\|\tau_t\|_2 + O(\|\nu_t\|_2).
\end{aligned}
\]

Since $\alpha_N<1$, iterating the first inequality gives
\[
\|\nu_t\|_2 \le C\,\alpha_N^t \|\delta_0\|_2.
\]
Substituting this bound into the second inequality yields
\[
\|\tau_t\|_2 = O\!\left((1+\varepsilon_T)^t\right)\|\delta_0\|_2,
\]
If $T_\mu\mathcal{M}=\{0\}$, all perturbations decay and $\mu^{(\ell)}$ is a point
attractor. Otherwise, contraction produces attraction
to a neighborhood of $\mathcal{M}$.
\end{proof}

\begin{remark}[Why don't transformers have global contraction?]
Global contraction $\alpha < 1$ at every layer would cause massive output collapse, destroying information. The conditionality of contraction to where it occurs near $\mu^{(\ell)}$ when attention concentrates, but not in a diverse context-rich input where attention can spread broadly. Basin trapping operates in a \textit{local trapping} schema ($\rho < 1$ in specific regions), meanwhile $\rho >1$ in global dynamics across layers to preserve information.
    
\end{remark}

\begin{theorem}[Trajectory trapping under a persistent contraction]
\label{thm:traj_trap}
Suppose there is a neighboring block of layers ($\ell_1,\dots,\ell_2$) such that:
\begin{enumerate}
    \item Each $f_\ell$ is radially contractive on $\mathcal{B}_\ell(r)$, i.e.\ $\bar{\alpha} < 1$,
    \item The trajectory enters the basin (reference region):
   $h^{(\ell_1)}(x) \in \mathcal{B}^{(\ell_1)}(r)$.
\end{enumerate}

Then for all $\ell \in [\ell_1,\ell_2]$,
\[
h^{(\ell)}(x) \in \mathcal{B}_\ell(r), \quad \| h^{(\ell)}(x) - \mu^{(\ell)} \|_2 \le \bar{\alpha}^{\ell - \ell_1} r.
\]
Thus the radial distance decays geometrically and the trajectory is effectively \textbf{trapped} because it cannot escape the contractive layers.
\end{theorem}
\begin{proof} We  prove  by induction. First, let us establish the base case holds: $\ell = \ell_1$ holds by assumption. For $\ell > \ell_1$, assume $h^{(\ell-1)}(x) \in \mathcal{B}^{(\ell-1)}(r)$ with $\| h^{(\ell-1)}(x) - \mu^{(\ell-1)} \|_2 \le \bar{\alpha}^{\ell-1-\ell_1} r$. Refer back to radial contraction, as now we can rearrange and evaluate:
\begin{align*}
\| h^{(\ell)}(x) - \mu^{(\ell)} \|_2 
&= \| f_{\ell-1}(h^{(\ell-1)}(x)) - \mu^{(\ell)} \|_2 \\
&\le \bar{\alpha} \| h^{(\ell-1)}(x) - \mu^{(\ell-1)} \|_2 \\
&\le \bar{\alpha} \cdot \bar{\alpha}^{\ell-1-\ell_1} r = \bar{\alpha}^{\ell-\ell_1} r.
\end{align*}
Thus $h^{(\ell)}(x) \in \mathcal{B}^{(\ell)}(r)$ and thus the bound holds.
\end{proof}

This result formalizes the trapping phenomenon: once a trajectory is captured, it cannot amplify any deviations in its readout to context-specific details.

\subsection{Task-Dependent Geometry}
We characterize basin geometry collapse via \textbf{variance collapse}: the ratio of hallucination to factual variance.

\begin{definition}[Variance ratio]
    For hidden states at a layer $\ell$, define:
\begin{align}
\rho_{\text{var}}^{(\ell)} &= \left(\sigma_{\text{fact}}^{(\ell)}\right)^2 \Big/ \left(\sigma_{\text{hall}}^{(\ell)}\right)^2, 
\text{where} \\  &\left(\sigma_c^{(\ell)}\right)^2 = \frac{1}{|C|} \sum_{i \in C} \left\|h_i^{(\ell)} - \mu_c^{(\ell)}\right\|_2^2 \nonumber
\end{align}
for class $c \in \{\text{fact, hallucination}\}$. Sharp basins exhibit $\rho_{\text{var}} \gg 1$, an indication that factual states occupy larger volume, meanwhile manifolds show $\rho_{\text{var}} \approx 1$, where both classes disperse due to dimensionality magnitude.
\end{definition}
\begin{theorem}[Task complexity determines basin geometry]
\label{thm:task_complexity}
Let $\mathcal{A}$ be the set of valid answers for a task. The variance ratio correlates with answer cardinality, such that:
\[
\rho_{\text{var}} = \frac{\text{Var}[\text{factual}]}{\text{Var}[\text{hallucinated}]} \geq C \log(|\mathcal{A}| + 1)
\]
where $C$ depends on embedding dimension and model capacity.
\end{theorem}

\begin{proof}[Proof idea]
Factoid tasks have unique correct answers, forcing hallucinated trajectories to collapse to task-independent reference states $\mu^{(\ell)}$ (constructed from uninformative contexts). The model has no semantic ``choice,'' yielding $\sigma^2_{\text{hall}} \ll \sigma^2_{\text{fact}}$. Generation tasks permit exponentially many valid summaries, preventing point convergence---both factual and hallucinated states explore the full embedding manifold, producing $\sigma^2_{\text{hall}} \approx \sigma^2_{\text{fact}}$. Misconception tasks retrieve confident but incorrect training memories (a dataset issue), geometrically indistinguishable from correct retrieval. Full proof in Appendix~\ref{app:task_complexity}. Table~\ref{tab:variance_analysis} confirms these values. \end{proof}

\begin{table*}[t]
\centering
\caption{\textbf{Variance Analysis by Task Type.} Factoid tasks show up to 4$\times$ variance expansion (factual states occupy larger volume), while summarization maintains parity (high-dimensional manifolds). Basin separation $d$ measured as $\|\mu_{\text{fact}} - \mu_{\text{hall}}\|_2$. $\rho_{\text{var}} = \text{Var}[\text{factual}]/\text{Var}[\text{hallucinated}]$.}
\label{tab:variance_analysis}
\small
\begin{tabular}{@{}llccc@{}}
\toprule
\textbf{Model} & \textbf{Dataset} & \textbf{Task Type} & $\boldsymbol{\rho_{\text{var}}}$ & \textbf{Basin Sep $d$} \\
\midrule
\multicolumn{5}{l}{\textit{Factoid: Point Attractors ($\rho_{\text{var}} \gg 1$)}} \\
Llama-1B & HaluEval QA & Factoid & 4.55 & 2.89 \\
Llama-1B & MuSiQue & Factoid & 10.00 & 3.40 \\
Qwen-1.5B & HaluEval QA & Factoid & 5.56 & 32.83 \\
Gemma-2B & HaluEval QA & Factoid & 1.82 & 58.91 \\
\midrule
\multicolumn{5}{l}{\textit{Summarization: High-Dimensional Manifolds ($\rho_{\text{var}} \approx 1$)}} \\
Llama-1B & Summarization & Generation & 1.45 & 0.49 \\
Gemma-2B & Summarization & Generation & 1.01 & 1.85 \\
\midrule
\multicolumn{5}{l}{\textit{Misconception: Competing Basins ($\rho_{\text{var}} \approx 1$-$1.4$)}} \\
Llama-1B & TruthfulQA & Misconception & 1.16 & 0.39 \\
Qwen-1.5B & TruthfulQA & Misconception & 1.39 & 4.20 \\
\bottomrule
\end{tabular}
\end{table*}

\subsection{Multi-Basin Partitioning}
For tasks with multiple plausible misconception style answers (e.g., TruthfulQA), the hallucinated state does not collapse into a single basin, but rather we observe that it partitions into distinct clusters.

\begin{theorem}[Multi-basin partitioning]
\label{thm:multi_basin}
Consider a task with $K$ common misconceptions. The hallucination subspace $\mathcal{H}^{(\ell)} = \{h^{(\ell)}(x) : x \in D_{\text{hall}}\}$ admits a Voronoi tessellation into $K$ basins centered at $\{\mu_1^{(\ell)}, \ldots, \mu_K^{(\ell)}\}$:
\[
\mathcal{B}_k^{(\ell)} = \left\{h \in \mathcal{H}^{(\ell)} : \|h - \mu_k^{(\ell)}\| \le \|h - \mu_j^{(\ell)}\| \forall j \neq k\right\}
\]
where each basin corresponds to a distinct misconception type with probability:
\[
P(\text{basin}_k | h) = \frac{\exp(-\|h - \mu_k\|^2 / 2\sigma^2)}{\sum_{j=1}^K \exp(-\|h - \mu_j\|^2 / 2\sigma^2)}.
\]    
\end{theorem}
\begin{proof}[Proof idea] Each misconception type, $m_k$, has a distinct semantic signature embedded in training data, creating local minima in the loss landscape at positions $\mu_k^{(\ell)}$ in layer $\ell$. Applying K-means clustering to hallucinated states $\mathcal{H}^{(\ell)}$ with $K$ centers yields basin centers $\{\mu_k^{(\ell)}\}_{k=1}^K$ that minimize within-cluster variance:
\[
\min_{\{\mu_k\}} \sum_{k=1}^K \sum_{h \in \mathcal{B}_k} \|h - \mu_k\|^2.
\]
The decision boundaries between basins are hyperplanes equidistant from adjacent centers, forming Voronoi cells \cite{1056489}. 
Each cell $\mathcal{B}_k$ captures trajectories that converge to misconception $m_k$. Full proof in App.\ \ref{app:multi_basin}. 
\end{proof}

\begin{remark}[Implications with hallucination detection]
Unlike single-basin tasks, where there is a clear separation between hallucinated and truthful outputs. The multi-basin setup is much more complicated in that there are clear overlaps with factual and hallucinatory states geometrically, as they try to retrieve confident yet distinct memories. This explains TruthfulQA's poor performance in basin detection.
\end{remark}

\subsection{Geometric Risk Metrics}
\label{sec:geo_risk}
This section defines three geometric metrics that can be evaluated
for a hidden state, $h^{(\ell)}(x)$, at a layer $\ell$.

\noindent \textbf{Distance to Reference State:} 
\label{def:distance}
The Euclidean distance of a hidden state $h^{(\ell)}$ to the nearest hallucination centroid $\mu^{(\ell)}$:
\[
d_{\text{basin}}^{(\ell)}(h) = \|h - \mu^{(\ell)}\|_2.
\]

\noindent \textbf{Class Separation:} To quantify geometric separations between factual and hallucinated distributions, we use the Fisher discriminant ratio that we define below.

\begin{definition}[Fisher separation ratio]
\label{def:fisher}
This metric measures the distances between classes, after being normalized. 
\[
\rho_{\text{Fisher}}^{(\ell)} = \frac{\|\mu_{\text{fact}}^{(\ell)} - \mu_{\text{hall}}^{(\ell)}\|^2_2}{\text{tr}(\Sigma_{\text{fact}}^{(\ell)}) + \text{tr}(\Sigma_{\text{hall}}^{(\ell)})},
\]
where $\mu_c^{(\ell)}, \Sigma_c^{(\ell)}$ are the mean and covariance of class $c \in \{\text{fact}, \text{hall}\}$ at layer $\ell$. High values of $\rho_{\text{Fisher}}$ indicate that basins are geometrically distinct and are linearly separable. This is because the metric just compares the distance between the latent activations at a layer, $\ell$. So a larger value means that the distances are significantly larger. The ratio quantifies inter-class distances normalized with in-class variance, similar to Mahalanobis distance \cite{6165382}.
\end{definition}
    
\section{Theoretical Properties}
Here we develop formal results from the hallucination basin construction. All results are in the latent geometry. 

\subsection{Basin Formations in $L$-Layer Transformers}

\begin{theorem}[$L$-layer basin emergence]
\label{thm:basin_emergence}
Assume the attention entropy is nearly-uniform at each layer ($H(\alpha^{(\ell)}) \geq H_0$). Then each layer formation propagates inductively.
\[
h^{(\ell)} \in \mathcal{B}^{(\ell)}(r_\ell) \implies h^{(\ell+1)} \in \mathcal{B}^{(\ell+1)}(r_{\ell+1})
\]
where $r_{\ell+1} \leq \alpha_\ell r_\ell$ with $\alpha_\ell < 1$.
\end{theorem}
\begin{proof}[Proof idea]
The core idea is that the Transformer layers act as this dynamic system. The Attention operator is the source of `expansion' normally (because it pulls in new information to allow the hidden state to move to new locations in its vector space). However, with a high entropy the model gets confused, and the attention mechanism ``gives up'' assigning equal weights to everything. Full proof in Appendix~\ref{app:basin_emergence}.
\end{proof}

\subsection{Radius Propagation}
We first characterize how the radius of a reference region propagates through layers under a persistent contraction.

\begin{proposition}[Radius decay]
If a trajectory enters $\mathcal{B}^{(\ell_0)}(r_0)$ at layer $\ell_0$ and all subsequent layers $\ell \in [\ell_0, L]$ are radially contractive with constant $\alpha_\ell \le \bar{\alpha} < 1$, then:
\[
r^{(\ell)}(x) \le \bar{\alpha}^{\ell - \ell_0} r_0, \quad \forall \ell \in [\ell_0, L].
\]
\end{proposition}
\begin{proof}
    By Def.~\ref{def:rad_contraction}, for any $h \in \mathcal{B}^{(\ell)}(r_\ell)$:
\[
\|\phi^{(\ell)}(h) - c^{(\ell+1)}\|_2 \le \alpha_\ell \|h - c^{(\ell)}\|_2 \le \alpha_\ell r_\ell.
\]

Through recursive application from $\ell_0$ to $\ell_i$ where $i$ is the iteration step, we get that $r^{(\ell)} \le \alpha_{\ell-1} r^{(\ell-1)} \le \cdots \le \prod_{k=\ell_0}^{\ell-1} \alpha_k \cdot r_0 \le \bar{\alpha}^{\ell - \ell_0} r_0$. This has an exponentially decaying factor, explaining why hallucinations become irreversible, as trajectory trapping propagates through the layers.
\end{proof}
\begin{corollary}[Asymptotic collapse]
From Thm.~\ref{thm:traj_trap},  we have that under a radial contraction $\bar{\alpha} < 1$, the basin radius actually vanishes as $r^{(\ell)} \to 0$ when $\ell \to \infty$.
\end{corollary}
This corollary implies how output distributions converge to a singular point. It also formalizes the intuition that subsequent input-specific information due to a geometric collapse from the trajectory trapping.

\subsubsection{Separation Lemma}

\begin{lemma}[Fact-hallucination separation]
\label{lemma:separation}
Assume factual inputs from task distribution $\mathcal{T}$ satisfy $\mathbb{E}_{x \sim \mathcal{T}}[\|h^{(\ell)}(x) - c^{(\ell)}\|_2] \ge \rho_\star$ at layer $\ell$ for some $\rho_\star > 0$. Then for basin radius $r < \rho_\star$:
\[
\mathbb{P}_{x \sim \mathcal{T}}[h^{(\ell)}(x) \in \mathcal{B}^{(\ell)}(r)] \le \tfrac{r}{\rho_\star}.
\]
\end{lemma}
\begin{proof}
Through Markov's inequality,   $\mathbb{P}[\|h^{(\ell)}(x) - c^{(\ell)}\|_2 \le r] \le \frac{\mathbb{E}[\|h^{(\ell)}(x) - c^{(\ell)}\|_2]}{r} \ge \frac{\rho_\star}{r}$. Thus factual trajectories avoid basins with probability $\ge 1 - r/\rho_\star$.

A direct rearrangement gives the result. 
\end{proof}

\section{An Adaptive Risk-Aware Steering Vector}

We have established basins as a geometrical structure behind LLM hallucinations. We now leverage them to create an intervention algorithm. 

We define a steering policy to intervene proportionally based on proximity to the nearest hallucination basin:
\[
h^{(\ell)}_{\text{steered}}(x) = h^{(\ell)}(x) + \lambda \cdot v^{(\ell)}_{\text{steer}},
\]
where the steering vector is computed as the difference between class centroids:
\[
v^{(\ell)}_{\text{steer}} = \frac{1}{|D_{\text{fact}}|} \sum_{x \in D_{\text{fact}}} h^{(\ell)}(x) - \frac{1}{|D_{\text{hall}}|} \sum_{x \in D_{\text{hall}}} h^{(\ell)}(x).
\]
The strength parameter $\lambda \in [0, 1]$ controls intervention intensity. More details in Appendix~\ref{alg:adaptive_steering}.

\section{Experiments}
We outline our validation protocol for the theoretical results. (1) validation of basin existence with a \textbf{quantifiable} geometric separation, (2) geometric features enable efficient detection without requiring sampling. View the experimental protocol in Appendix \ref{app:evaluation_protocol}.

\subsection{Experimental Design and Setup}
\paragraph{Models.} To demonstrate generalizability and scales we evaluate on: Llama 3.2-1B/3B \cite{meta2024llama32}, Gemma-2-2B \cite{team2024gemma} and Qwen2-1.5B \cite{yang2025qwen3}.

\paragraph{Datasets.} We use four diverse hallucination benchmarks:
HaluEval \cite{li2023halueval},
MuSiQue \cite{trivedi2022musique},
FEVER \cite{thorne-etal-2018-fever}, and TruthfulQA \cite{DBLP:journals/corr/abs-2109-07958}.

\paragraph{Hidden State Extraction} We use autoregressive decoding trajectories and extract final-token hidden states layerwise in a 70/30 stratified split with $\text{seed}=42.$

\subsection{Task-Dependent Basin Formation}

\paragraph{Hypothesis:} We test whether basin geometry under autoregressive decoding remains task-dependent: factoid settings should be more separable, while generation and misconception settings should show weaker or overlapping structure. Table~\ref{tab:basin-matrix} and Figure~\ref{fig:task_dependent} summarize the evidence.

\begin{figure}[!h]
\centering
\includegraphics[width=0.95\columnwidth]{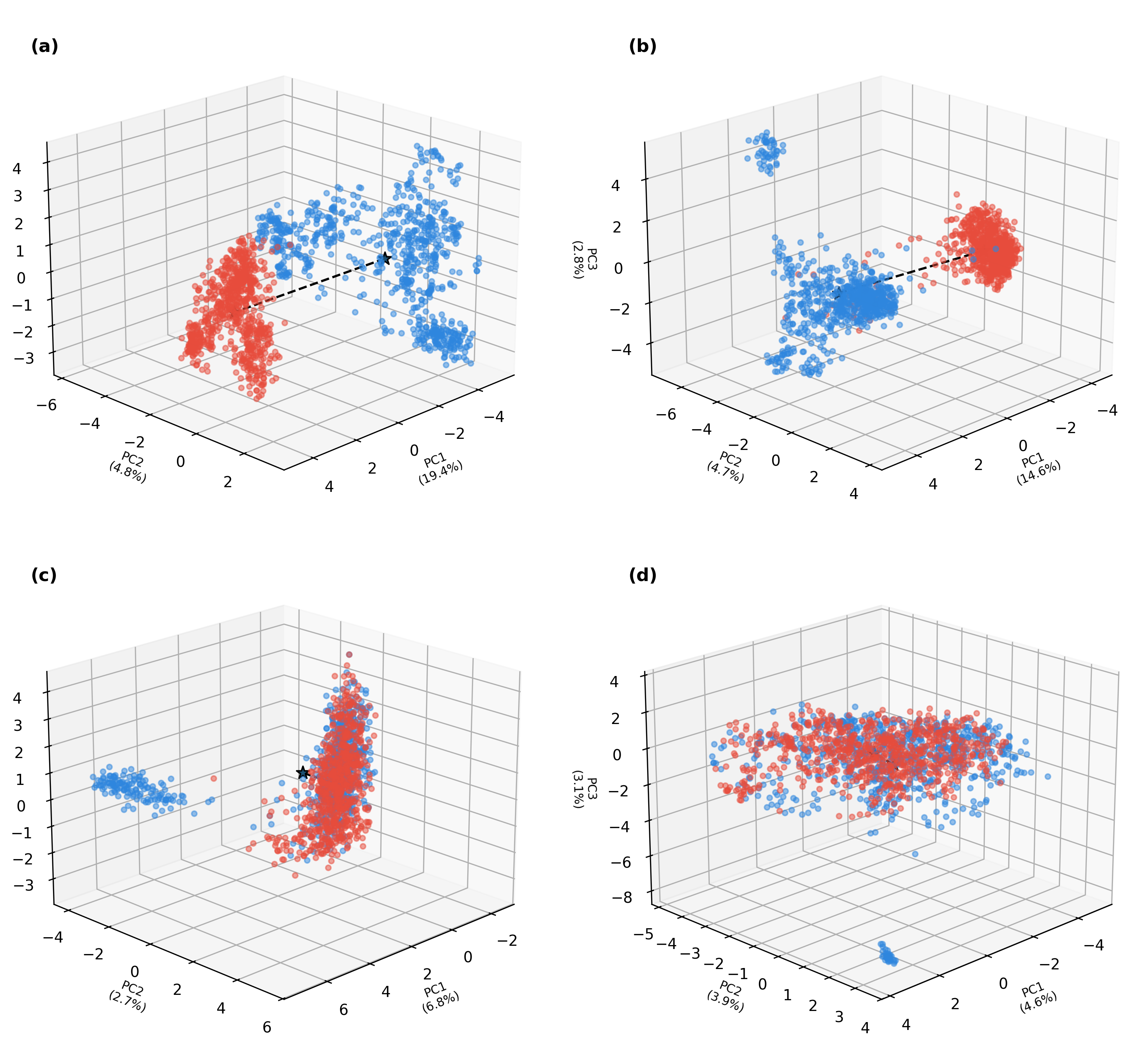}
\caption{\textbf{Task-Dependent Basin Geometry.} Llama-3.2-3b's performance on various tasks and 3D PCA projected outputs. (a) shows performance on MuSiQue, (b) shows performance on HaluEvalQA, (c) shows performance on HaluEvalSummarization, (d) shows performance on TruthfulQA.  }
\vspace{-0.8 cm}
\label{fig:task_dependent}
\end{figure}

\begin{table*}[t]
\centering
\small
\caption{Centroid and Mahalanobis (Maha) AUROC reported with 95\% CI. Note: Lay indicates Layer with highest AUROC, $N$ indicates number of data samples, and $B?$ indicates whether a basin exists or not. For larger models, due to computational constraints, this was limited to $N=1000$.}
\label{tab:basin-matrix}
\begin{tabular}{@{}lllccc c c@{}}
\toprule
Model & Data & Lay & Centroid (95\% CI) & Maha (95\% CI) & $N$ & B? \\
\midrule
gemma-2-2b & FEVER & $L_{14}$ & $0.515\ (0.489,0.548)$ & $0.514\ (0.487,0.535)$ & $9999$ & \texttimes \\
gemma-2-2b & HaluEval\_qa & $L_{14}$ & $0.727\ (0.703,0.744)$ & $0.725\ (0.710,0.745)$ & $20000$ & \checkmark \\
gemma-2-2b & HaluEval\_summ & $L_{20}$ & $0.508\ (0.491,0.519)$ & $0.479\ (0.457,0.502)$ & $20000$ & \texttimes \\
gemma-2-2b & MuSiQue & $L_{26}$ & $0.912\ (0.894,0.932)$ & $0.926\ (0.909,0.944)$ & $4834$ & \checkmark \\
gemma-2-2b & TruthfulQA & $L_{14}$ & $0.607\ (0.547,0.662)$ & $0.597\ (0.535,0.651)$ & $1580$ & \texttimes \\
llama-3.2-1b & FEVER & $L_{8}$ & $0.670\ (0.641,0.700)$ & $0.680\ (0.659,0.702)$ & $9999$ & \texttimes \\
llama-3.2-1b & HaluEval\_qa & $L_{3}$ & $0.983\ (0.976,0.988)$ & $0.984\ (0.980,0.988)$ & $20000$ & \checkmark \\
llama-3.2-1b & HaluEval\_summ & $L_{10}$ & $0.681\ (0.666,0.697)$ & $0.674\ (0.659,0.690)$ & $20000$ & \texttimes \\
llama-3.2-1b & MuSiQue & $L_{1}$ & $1.000\ (1.000,1.000)$ & $1.000\ (1.000,1.000)$ & $4834$ & \checkmark \\
llama-3.2-1b & TruthfulQA & $L_{12}$ & $0.741\ (0.662,0.800)$ & $0.724\ (0.685,0.777)$ & $1580$ & \checkmark \\
llama-3.2-3b & FEVER & $L_{12}$ & $0.702\ (0.671,0.725)$ & $0.711\ (0.686,0.731)$ & $9999$ & \checkmark \\
llama-3.2-3b & HaluEval\_qa & $L_{3}$ & $0.986\ (0.982,0.990)$ & $0.985\ (0.981,0.990)$ & $20000$ & \checkmark \\
llama-3.2-3b & HaluEval\_summ & $L_{21}$ & $0.669\ (0.654,0.687)$ & $0.665\ (0.648,0.683)$ & $20000$ & \texttimes \\
llama-3.2-3b & MuSiQue & $L_{3}$ & $1.000\ (1.000,1.000)$ & $1.000\ (1.000,1.000)$ & $4834$ & \checkmark \\
llama-3.2-3b & TruthfulQA & $L_{12}$ & $0.771\ (0.716,0.833)$ & $0.794\ (0.751,0.839)$ & $1580$ & \checkmark \\
qwen-2.5-1.5b & FEVER & $L_{18}$ & $0.728\ (0.704,0.748)$ & $0.735\ (0.719,0.757)$ & $9999$ & \checkmark \\
qwen-2.5-1.5b & HaluEval\_qa & $L_{24}$ & $0.984\ (0.979,0.989)$ & $0.983\ (0.980,0.988)$ & $20000$ & \checkmark \\
qwen-2.5-1.5b & HaluEval\_summ & $L_{18}$ & $0.663\ (0.650,0.683)$ & $0.664\ (0.648,0.682)$ & $20000$ & \texttimes \\
qwen-2.5-1.5b & MuSiQue & $L_{3}$ & $1.000\ (1.000,1.000)$ & $1.000\ (1.000,1.000)$ & $4834$ & \checkmark \\
qwen-2.5-1.5b & TruthfulQA & $L_{21}$ & $0.738\ (0.671,0.803)$ & $0.751\ (0.705,0.803)$ & $1580$ & \checkmark \\
llama-3.1-8b & HaluEval\_qa & $L_{0}$ & $0.571\ (0.503,0.731)$ & $0.549\ (0.503,0.705)$ & $1000$ & \texttimes \\
llama-3.1-8b & TruthfulQA & $L_{25}$ & $0.944\ (0.899,0.975)$ & $0.958\ (0.921,0.987)$ & $1000$ & \checkmark \\
mistral-7b-v0.3 & HaluEval\_qa & $L_{24}$ & $0.704\ (0.578,0.823)$ & $0.545\ (0.503,0.701)$ & $1000$ & \texttimes \\
mistral-7b-v0.3 & TruthfulQA & $L_{17}$ & $0.939\ (0.893,0.975)$ & $0.958\ (0.923,0.985)$ & $1000$ & \checkmark \\
\bottomrule
\end{tabular}
\end{table*}

\subsection{Causality: Pushing Factual $\to$ Basins}

\paragraph{Method}Linearly interpolate factual hidden states toward basin centroid: $h_\alpha = (1-\alpha) h_{\text{fact}} + \alpha \mu_{\text{hall}}$ for $\alpha \in [0, 1]$. Train logistic classifier on factual/hall, measure P(hall$|h_\alpha$).

\begin{figure}[!h]
\centering
\includegraphics[width=0.95\columnwidth]{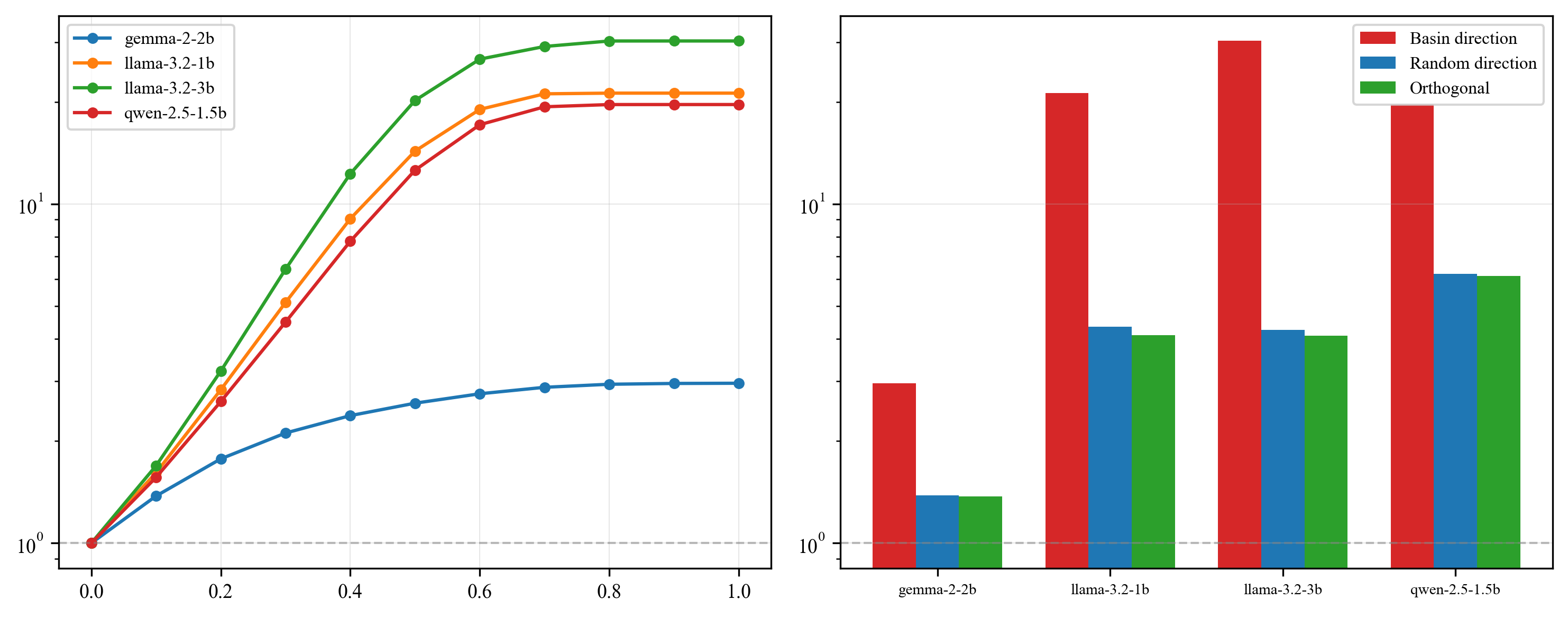}
\caption{\textbf{Causal Intervention: Factual $\to$ Basin.} (Left) Dose-response curve fold increase in hallucination probability as factual hidden states are in-model steered toward the hallucination centroid (interpolation strength $\alpha$ on the horizontal axis). Right: bar plot comparing the maximum fold increase produced by steering along the basin direction versus two controls (random direction and an orthogonal direction). See Appendix \ref{app:intervention} and \ref{app:discussion}.}
\label{fig:causality}
\end{figure}

\section{Discussion and Remarks}
\paragraph{When Basins Don't Form} Table~\ref{tab:basin-matrix} reveals a systematic trend of failures in basin formations. Notice that in TruthfulQA and summarization the AUROC value lingers between 0.5 across all models, indicating a near random performance. 

\paragraph{Misconception Tasks} TruthfulQA contains common misconceptions (e.g., ``What happens if you crack your knuckles?") where models confidently retrieve incorrect training data. These create \textit{multiple indistinguishable basins}. Both factual and hallucinated states converge to confident retrieval modes, preventing geometric separation. Our theory assumes hallucinations collapse to task-independent reference states, which fails when confident, incorrect memories exist.

\paragraph{Architectural Variations} Gemma-2B uses GroupedQueryAttention and different LayerNorm placement compared to Llama/Qwen architectures. Architectural deviations may alter spectral properties, requiring model-specific analysis.

\paragraph{Limitations}
We further discuss the limitations in Appendix~\ref{limitations}.

\bibliography{refs}

\begin{thebibliography}{34}
\providecommand{\natexlab}[1]{#1}
\providecommand{\url}[1]{\texttt{#1}}
\expandafter\ifx\csname urlstyle\endcsname\relax
  \providecommand{\doi}[1]{doi: #1}\else
  \providecommand{\doi}{doi: \begingroup \urlstyle{rm}\Url}\fi

\bibitem[Alansari \& Luqman(2025)Alansari and Luqman]{alansari2025large}
Alansari, A. and Luqman, H.
\newblock Large language models hallucination: A comprehensive survey.
\newblock arXiv:2510.06265, 2025.

\bibitem[Chen \& Zhang(2025)Chen and Zhang]{CHEN2025128893}
Chen, B. and Zhang, H.
\newblock High-order rotor {H}opfield neural networks for associative memory.
\newblock \emph{Neurocomputing}, 616:\penalty0 128893, 2025.
\newblock \doi{10.1016/j.neucom.2024.128893}.

\bibitem[Chen et~al.(2024{\natexlab{a}})Chen, Liu, Chen, Gu, Wu, Tao, Fu, and Ye]{chen2024inside}
Chen, C., Liu, K., Chen, Z., Gu, Y., Wu, Y., Tao, M., Fu, Z., and Ye, J.
\newblock {INSIDE}: {LLMs'} internal states retain the power of hallucination detection.
\newblock arXiv:2402.03744, 2024{\natexlab{a}}.

\bibitem[Chen et~al.(2024{\natexlab{b}})Chen, Xiong, Liu, Wu, Xiao, Gao, and He]{pmlr-v235-chen24av}
Chen, S., Xiong, M., Liu, J., Wu, Z., Xiao, T., Gao, S., and He, J.
\newblock In-context sharpness as alerts: An inner representation perspective for hallucination mitigation.
\newblock In \emph{Proceedings of the 41st International Conference on Machine Learning}, pp.\  7553--7567, 2024{\natexlab{b}}.

\bibitem[Essex et~al.(2025)Essex, Janson, Norris, and Balanov]{essex2025memorisation}
Essex, A.~E., Janson, N.~B., Norris, R.~A., and Balanov, A.~G.
\newblock Memorisation and forgetting in a learning {H}opfield neural network: bifurcation mechanisms, attractors and basins.
\newblock arXiv:2508.10765, 2025.

\bibitem[Farquhar et~al.(2024)Farquhar, Kossen, Kuhn, and Gal]{Farquhar2024SemanticEntropy}
Farquhar, S., Kossen, J., Kuhn, L., and Gal, Y.
\newblock Detecting hallucinations in large language models using semantic entropy.
\newblock \emph{Nature}, 630:\penalty0 625--630, 2024.
\newblock \doi{10.1038/s41586-024-07421-0}.

\bibitem[Hopfield(1982)]{doi:10.1073/pnas.79.8.2554}
Hopfield, J.~J.
\newblock Neural networks and physical systems with emergent collective computational abilities.
\newblock \emph{Proceedings of the National Academy of Sciences}, 79\penalty0 (8):\penalty0 2554--2558, 1982.
\newblock \doi{10.1073/pnas.79.8.2554}.

\bibitem[Huang et~al.(2025)Huang, Yu, Ma, Zhong, Feng, Wang, Chen, Peng, Feng, Qin, and Liu]{huang2025survey}
Huang, L., Yu, W., Ma, W., Zhong, W., Feng, Z., Wang, H., Chen, Q., Peng, W., Feng, X., Qin, B., and Liu, T.
\newblock A survey on hallucination in large language models: Principles, taxonomy, challenges, and open questions.
\newblock \emph{ACM Transactions on Information Systems}, 43\penalty0 (2):\penalty0 42, 2025.
\newblock \doi{10.1145/3703155}.

\bibitem[Inazawa(2025)]{inazawa2025associative}
Inazawa, H.
\newblock Associative memory model with neural networks: Memorizing multiple images with one neuron.
\newblock arXiv:2510.06542, 2025.

\bibitem[Kafraj et~al.(2025)Kafraj, Krotov, Bicknell, and Latham]{kafraj2025a}
Kafraj, M.~S., Krotov, D., Bicknell, B.~A., and Latham, P.~E.
\newblock A biologically plausible associative memory network.
\newblock In \emph{ICLR 2025 Workshop on New Frontiers in Associative Memories}, 2025.
\newblock URL \url{https://openreview.net/forum?id=u4YzOzEMfR}.

\bibitem[Karbasi et~al.(2025)Karbasi, Montasser, Sous, and Velegkas]{karbasi2025impossibility}
Karbasi, A., Montasser, O., Sous, J., and Velegkas, G.
\newblock {(Im)}possibility of automated hallucination detection in large language models.
\newblock In \emph{ICML 2025 Workshop on Reliable and Responsible Foundation Models}, 2025.
\newblock URL \url{https://openreview.net/forum?id=B4SFmNvBNz}.

\bibitem[Li et~al.(2023)Li, Cheng, Zhao, Nie, and Wen]{li2023halueval}
Li, J., Cheng, X., Zhao, W.~X., Nie, J.-Y., and Wen, J.-R.
\newblock {HaluEval}: A large-scale hallucination evaluation benchmark for large language models.
\newblock arXiv:2305.11747, 2023.

\bibitem[Li et~al.(2025)Li, Luo, Zhang, and Liu]{LI2025116657}
Li, X., Luo, M., Zhang, B., and Liu, S.
\newblock Dynamic analysis and implementation of a multi-stable {H}opfield neural network.
\newblock \emph{Chaos, Solitons \& Fractals}, 199:\penalty0 116657, 2025.
\newblock \doi{10.1016/j.chaos.2025.116657}.

\bibitem[Lin et~al.(2022)Lin, Hilton, and Evans]{DBLP:journals/corr/abs-2109-07958}
Lin, S., Hilton, J., and Evans, O.
\newblock {TruthfulQA}: Measuring how models mimic human falsehoods.
\newblock In \emph{Proceedings of the 60th Annual Meeting of the Association for Computational Linguistics}, pp.\  3214--3252, 2022.

\bibitem[Lloyd(1982)]{1056489}
Lloyd, S.
\newblock Least squares quantization in {PCM}.
\newblock \emph{IEEE Transactions on Information Theory}, 28\penalty0 (2):\penalty0 129--137, 1982.
\newblock \doi{10.1109/TIT.1982.1056489}.

\bibitem[Manakul et~al.(2023)Manakul, Liusie, and Gales]{manakul2023selfcheckgpt}
Manakul, P., Liusie, A., and Gales, M.
\newblock {SelfCheckGPT}: Zero-resource black-box hallucination detection for generative large language models.
\newblock In \emph{Proceedings of the 2023 Conference on Empirical Methods in Natural Language Processing}, pp.\  9004--9017, 2023.

\bibitem[{Meta AI}(2024)]{meta2024llama32}
{Meta AI}.
\newblock The {Meta} {Llama} 3.2 collection of multilingual language models.
\newblock \url{https://github.com/meta-llama/llama-models/blob/main/models/llama3_2/MODEL_CARD.md}, 2024.
\newblock Accessed: 2026-01-15.

\bibitem[Park et~al.(2025)Park, Du, Yeh, Wang, and Li]{park2025steer}
Park, S., Du, X., Yeh, M.-H., Wang, H., and Li, Y.
\newblock Steer {LLM} latents for hallucination detection.
\newblock In \emph{Proceedings of the Forty-second International Conference on Machine Learning}, pp.\  47971--47990, 2025.

\bibitem[Pezzulo et~al.(2021)Pezzulo, LaPalme, Durant, and Levin]{10.1098/rstb.2019.0765}
Pezzulo, G., LaPalme, J., Durant, F., and Levin, M.
\newblock Bistability of somatic pattern memories: stochastic outcomes in bioelectric circuits underlying regeneration.
\newblock \emph{Philosophical Transactions of the Royal Society B: Biological Sciences}, 376\penalty0 (1821):\penalty0 20190765, 2021.
\newblock \doi{10.1098/rstb.2019.0765}.

\bibitem[Ramsauer et~al.(2021)Ramsauer, Sch{\"{a}}fl, Lehner, Seidl, Widrich, Gruber, Holzleitner, Adler, Kreil, Kopp, Klambauer, Brandstetter, and Hochreiter]{Ramsauer_ea2021}
Ramsauer, H., Sch{\"{a}}fl, B., Lehner, J., Seidl, P., Widrich, M., Gruber, L., Holzleitner, M., Adler, T., Kreil, D., Kopp, M.~K., Klambauer, G., Brandstetter, J., and Hochreiter, S.
\newblock Hopfield networks is all you need.
\newblock In \emph{Proceedings of the 9th International Conference on Learning Representations}, 2021.

\bibitem[Riviere et~al.(2024)Riviere, Pathak, Sessa, Hardin, Bhupatiraju, Hussenot, Mesnard, Shahriari, Ram{\'e}, et~al.]{team2024gemma}
Riviere, M., Pathak, S., Sessa, P.~G., Hardin, C., Bhupatiraju, S., Hussenot, L., Mesnard, T., Shahriari, B., Ram{\'e}, A., et~al.
\newblock Gemma 2: Improving open language models at a practical size.
\newblock arXiv:2408.00118, 2024.

\bibitem[Sahoo et~al.(2024)Sahoo, Saxena, Maharaj, Ahmad, Mishra, and Bhattacharyya]{sahoo2024addressing}
Sahoo, N.~R., Saxena, A., Maharaj, K., Ahmad, A.~A., Mishra, A., and Bhattacharyya, P.
\newblock Addressing bias and hallucination in large language models.
\newblock In \emph{Proceedings of the 2024 Joint International Conference on Computational Linguistics, Language Resources and Evaluation}, pp.\  73--79, 2024.

\bibitem[Sriramanan et~al.(2024)Sriramanan, Bharti, Sadasivan, Saha, Kattakinda, and Feizi]{sriramanan2024llm}
Sriramanan, G., Bharti, S., Sadasivan, V.~S., Saha, S., Kattakinda, P., and Feizi, S.
\newblock {LLM-Check}: Investigating detection of hallucinations in large language models.
\newblock In \emph{Advances in Neural Information Processing Systems}, volume~37, pp.\  34188--34216. 2024.

\bibitem[Sun et~al.(2025{\natexlab{a}})Sun, Gai, Chen, Ravichander, Choi, Dziri, and Song]{sun2025why}
Sun, Y., Gai, Y., Chen, L., Ravichander, A., Choi, Y., Dziri, N., and Song, D.
\newblock Why and how {LLM}s hallucinate: Connecting the dots with subsequence associations.
\newblock In \emph{Advances in Neural Information Processing Systems}, volume~37, pp.\  34188--34216. 2025{\natexlab{a}}.

\bibitem[Sun et~al.(2025{\natexlab{b}})Sun, Ochiai, Wu, Lin, and Kanai]{sun2025associative}
Sun, Y., Ochiai, H., Wu, Z., Lin, S., and Kanai, R.
\newblock Associative transformer.
\newblock In \emph{Proceedings of the IEEE/CVF Computer Vision and Pattern Recognition Conference}, pp.\  4518--4527, 2025{\natexlab{b}}.

\bibitem[Sun et~al.(2024)Sun, Zang, Zheng, Song, Xu, Zhang, Yu, and Li]{sun2024redeep}
Sun, Z., Zang, X., Zheng, K., Song, Y., Xu, J., Zhang, X., Yu, W., and Li, H.
\newblock {ReDeEP}: Detecting hallucination in retrieval-augmented generation via mechanistic interpretability.
\newblock arXiv:2410.11414, 2024.

\bibitem[Tan et~al.(2025)Tan, Huang, Shi, and Wu]{tan2025interpdetect}
Tan, L., Huang, K.-W., Shi, J., and Wu, K.
\newblock {InterpDetect}: Interpretable signals for detecting hallucinations in retrieval-augmented generation.
\newblock arXiv:2510.21538, 2025.

\bibitem[Thorne et~al.(2018)Thorne, Vlachos, Christodoulopoulos, and Mittal]{thorne-etal-2018-fever}
Thorne, J., Vlachos, A., Christodoulopoulos, C., and Mittal, A.
\newblock {FEVER}: a large-scale dataset for {F}act {E}xtraction and {VER}ification.
\newblock In Walker, M., Ji, H., and Stent, A. (eds.), \emph{Proceedings of the 2018 Conference of the North {A}merican Chapter of the Association for Computational Linguistics}, pp.\  809--819, June 2018.
\newblock \doi{10.18653/v1/N18-1074}.

\bibitem[Trivedi et~al.(2022)Trivedi, Balasubramanian, Khot, and Sabharwal]{trivedi2022musique}
Trivedi, H., Balasubramanian, N., Khot, T., and Sabharwal, A.
\newblock {MuSiQue}: Multihop questions via single-hop question composition.
\newblock \emph{Transactions of the Association for Computational Linguistics}, 10:\penalty0 539--554, 2022.
\newblock \doi{10.1162/tacl_a_00475}.

\bibitem[Varshney(2012)]{6165382}
Varshney, K.~R.
\newblock Generalization error of linear discriminant analysis in spatially-correlated sensor networks.
\newblock \emph{IEEE Transactions on Signal Processing}, 60\penalty0 (6):\penalty0 3295--3301, 2012.
\newblock \doi{10.1109/TSP.2012.2190063}.

\bibitem[Wang et~al.(2025)Wang, Jiao, Zhu, Chen, He, Chu, Gao, Wang, and Ma]{wang2025adaptive}
Wang, T., Jiao, X., Zhu, Y., Chen, Z., He, Y., Chu, X., Gao, J., Wang, Y., and Ma, L.
\newblock Adaptive activation steering: A tuning-free {LLM} truthfulness improvement method for diverse hallucinations categories.
\newblock In \emph{Proceedings of the ACM on Web Conference 2025}, pp.\  2562--2578, 2025.

\bibitem[Yang et~al.(2025)Yang, Li, Yang, Zhang, Hui, Zheng, Yu, Gao, Huang, Lv, et~al.]{yang2025qwen3}
Yang, A., Li, A., Yang, B., Zhang, B., Hui, B., Zheng, B., Yu, B., Gao, C., Huang, C., Lv, C., et~al.
\newblock Qwen3 technical report.
\newblock arXiv:2505.09388, 2025.

\bibitem[Zhu et~al.(2025)Zhu, Liu, Zhang, Wang, Chen, Wang, Luo, Zhang, et~al.]{zhualleviating}
Zhu, C., Liu, Y., Zhang, H., Wang, A., Chen, G., Wang, L., Luo, W., Zhang, K., et~al.
\newblock Alleviating hallucinations in large language models through multi-model contrastive decoding and dynamic hallucination detection.
\newblock In \emph{Advances in Neural Information Processing Systems}, volume~39. 2025.

\bibitem[Zou et~al.(2025)Zou, Wang, Yan, Lyu, Zheng, Huang, Chen, Jiang, Liu, Tang, and Hu]{zou2025look}
Zou, X., Wang, Y., Yan, Y., Lyu, Y., Zheng, K., Huang, S., Chen, J., Jiang, P., Liu, J., Tang, C., and Hu, X.
\newblock Look twice before you answer: Memory-space visual retracing for hallucination mitigation in multimodal large language models.
\newblock In \emph{Proceedings of the 42nd International Conference on Machine Learning}, pp.\  80873--80899, 2025.

\end{thebibliography}
\bibliographystyle{icml2026}

\newpage
\appendix
\onecolumn
\section*{Appendix for \textit{Hallucination Basins: A Dynamic Framework for Understanding and Controlling LLM Hallucinations}}

\section{Full Proofs}

\FloatBarrier
\subsection{Proof of Theorem \ref{thm:task_complexity}}
\label{app:task_complexity}
\begin{theorem}[Task complexity determines basin geometry]
Let $\mathcal{A}$ be the set of valid answers for a task. The variance ratio correlates with answer cardinality, such that:
\[
\rho_{\text{var}} = \frac{\text{Var}[\text{factual}]}{\text{Var}[\text{hallucinated}]} \geq C \log(|\mathcal{A}| + 1)
\]
where $C$ depends on embedding dimension and model capacity.
\end{theorem}
\begin{proof}
Representing $|\mathcal{A}|$ distinct answers requires intrinsic dimensionality. we introduce three assumptions for this proof.
\begin{assumption}[Signal dimension] The factual hidden states lie in an intrinsic signal subspace of dimension \(d_{\mathrm{signal}}\) and per-coordinate signal variance at least \(\sigma_{\mathrm{sig}}^2>0\); hence
     \[
    \operatorname{Var}[\mathrm{factual}] \;\ge\; d_{\mathrm{signal}}\sigma_{\mathrm{sig}}^2.
  \]
\end{assumption}
\begin{assumption}[Hallucination noise.] Hallucinated states concentrate around a reference \(\mu_{\mathrm{ref}}\) with residual isotropic noise variance \(\sigma_0^2\), and the effective noise dimension is bounded by \(d_{\mathrm{hall}}\) (often extremely small for point-attractor collapse), so
  \[
    \operatorname{Var}[\mathrm{hallucinated}] \;\le\; d_{\mathrm{hall}}\sigma_0^2.
  \]
\end{assumption}
\begin{assumption}[Encoding lower bound] Representing \(|\mathcal{A}|\) distinct answers requires intrinsic dimension at least
  \[
    d_{\mathrm{signal}} \;\ge\; \log_2(|\mathcal{A}|+1).
  \]    
\end{assumption}

Note that for the last assumption, each extra bit of dimensionality in the representation doubles the number of reliably separate states in the ideal model's quantization.
From Assumptions A.1 and A.3, 
\[
  \operatorname{Var}[\mathrm{factual}] \;\ge\; d_{\mathrm{signal}}\sigma_{\mathrm{sig}}^2
  \;\ge\; \sigma_{\mathrm{sig}}^2\log_2(|\mathcal{A}|+1).
\]
From Assumption A.2,
\[
  \operatorname{Var}[\mathrm{hallucinated}] \;\le\; d_{\mathrm{hall}}\sigma_0^2.
\]
Hence
\[
  \rho_{\mathrm{var}}
  \;=\; \frac{\operatorname{Var}[\mathrm{factual}]}{\operatorname{Var}[\mathrm{hallucinated}]}
  \;\ge\; \frac{\sigma_{\mathrm{sig}}^2\log_2(|\mathcal{A}|+1)}{d_{\mathrm{hall}}\sigma_0^2}.
\]
Define the model-dependent constant
\[
  C := \frac{\sigma_{\mathrm{sig}}^2}{d_{\mathrm{hall}}\sigma_0^2}.
\]
Then the bound becomes
\[
  \rho_{\mathrm{var}} \ge C \log_2(|\mathcal{A}|+1),
\]
which is equivalent to the theorem's stated form.
\end{proof}

\subsection{Proof of Theorem \ref{thm:basin_emergence}}
\label{app:basin_emergence}
\begin{theorem}[$L$-layer basin emergence]
Assume the attention entropy is nearly-uniform at each layer ($H(\alpha^{(\ell)}) \geq H_0$). Then each layer formation propagates inductively.
\[
h^{(\ell)} \in \mathcal{B}^{(\ell)}(r_\ell) \implies h^{(\ell+1)} \in \mathcal{B}^{(\ell+1)}(r_{\ell+1})
\]
where $r_{\ell+1} \leq \alpha_\ell r_\ell$ with $\alpha_\ell < 1$.
\end{theorem}

\begin{proof}
We introduce these assumptions for the proof.
\begin{assumption}[Layer structure]
\label{assump:layer_structure}
Each Transformer layer $\ell$ computes
\[
h^{(\ell+1)} \;=\; \mathrm{LN}\!\left(
h^{(\ell)} + \mathrm{Attn}^{(\ell)}(h^{(\ell)}) + \mathrm{FFN}^{(\ell)}(h^{(\ell)})
\right),
\]
where $\mathrm{LN}$ denotes Layer Normalization applied after the residual sum.
\end{assumption}

\begin{assumption}[Near-uniform attention]
\label{assump:uniform_attention}
There exists $\varepsilon_\ell > 0$ such that for all
$h \in \mathcal{B}^{(\ell)}(r_\ell)$,
the attention weights satisfy
\[
\left\|\alpha^{(\ell)}(h) - \tfrac{1}{n}\mathbf{1}\right\|_\infty \le \varepsilon_\ell,
\]
which is implied by the entropy condition
$H(\alpha^{(\ell)}) \ge H_0$.
\end{assumption}

\begin{assumption}[Residual branch Lipschitzness]
\label{assump:residual_lipschitz}
There exist constants $L_A, L_F \ge 0$ such that for all
$h \in \mathcal{B}^{(\ell)}(r_\ell)$,
\begin{align*}
\|\mathrm{Attn}^{(\ell)}(h) - \mathrm{Attn}^{(\ell)}(\mu^{(\ell)})\|
&\le L_A \|h - \mu^{(\ell)}\| + C_A \varepsilon_\ell, \\
\|\mathrm{FFN}^{(\ell)}(h) - \mathrm{FFN}^{(\ell)}(\mu^{(\ell)})\|
&\le L_F \|h - \mu^{(\ell)}\|.
\end{align*}
\end{assumption}

\begin{assumption}[LayerNorm contraction]
\label{assump:layernorm}
Layer Normalization is locally Lipschitz with constant
$L_{\mathrm{LN}} < 1$ on the image of $\mathcal{B}^{(\ell)}(r_\ell)$, i.e.
\[
\|\mathrm{LN}(x) - \mathrm{LN}(y)\| \le L_{\mathrm{LN}} \|x - y\|
\quad \text{for all relevant } x,y.
\]
\end{assumption}

\begin{assumption}[Centroid consistency]
\label{assump:centroid}
The basin centers propagate according to the layer map:
\[
\mu^{(\ell+1)} = \mathrm{LN}\!\left(
\mu^{(\ell)} + \mathrm{Attn}^{(\ell)}(\mu^{(\ell)}) + \mathrm{FFN}^{(\ell)}(\mu^{(\ell)})
\right).
\]
\end{assumption}

Let $h^{(\ell)} \in \mathcal{B}^{(\ell)}(r_\ell)$.
By Assumption~\ref{assump:layer_structure} and
Assumption~\ref{assump:centroid},
\[
h^{(\ell+1)} - \mu^{(\ell+1)}
=
\mathrm{LN}(z(h^{(\ell)})) - \mathrm{LN}(z(\mu^{(\ell)})),
\]
where
\[
z(h) := h + \mathrm{Attn}^{(\ell)}(h) + \mathrm{FFN}^{(\ell)}(h).
\]

Applying the Lipschitz property of LayerNorm
(Assumption~\ref{assump:layernorm}),
\[
\|h^{(\ell+1)} - \mu^{(\ell+1)}\|
\le
L_{\mathrm{LN}} \|z(h^{(\ell)}) - z(\mu^{(\ell)})\|.
\]

We expand the residual difference:
\begin{align*}
\|z(h^{(\ell)}) - z(\mu^{(\ell)})\|
&\le
\|h^{(\ell)} - \mu^{(\ell)}\| \\
&\quad
+ \|\mathrm{Attn}^{(\ell)}(h^{(\ell)}) - \mathrm{Attn}^{(\ell)}(\mu^{(\ell)})\| \\
&\quad
+ \|\mathrm{FFN}^{(\ell)}(h^{(\ell)}) - \mathrm{FFN}^{(\ell)}(\mu^{(\ell)})\|.
\end{align*}

Using Assumption~\ref{assump:residual_lipschitz},
\[
\|z(h^{(\ell)}) - z(\mu^{(\ell)})\|
\le
(1 + L_A + L_F)\|h^{(\ell)} - \mu^{(\ell)}\|
+ C_A \varepsilon_\ell.
\]

Substituting into the LayerNorm bound yields
\[
\|h^{(\ell+1)} - \mu^{(\ell+1)}\|
\le
L_{\mathrm{LN}}(1 + L_A + L_F)\|h^{(\ell)} - \mu^{(\ell)}\|
+ L_{\mathrm{LN}} C_A \varepsilon_\ell.
\]

Define
\[
\alpha_\ell := L_{\mathrm{LN}}(1 + L_A + L_F),
\qquad
C_\ell := L_{\mathrm{LN}} C_A \varepsilon_\ell.
\]

Since $L_{\mathrm{LN}} < 1$ and the residual Lipschitz constants are finite,
we may choose the basin radius $r_\ell$ and entropy threshold $H_0$
so that $\alpha_\ell < 1$ and $C_\ell \ll r_\ell$.

Therefore, for all $h^{(\ell)} \in \mathcal{B}^{(\ell)}(r_\ell)$,
\[
\|h^{(\ell+1)} - \mu^{(\ell+1)}\|
\le \alpha_\ell r_\ell,
\]
which implies
\[
h^{(\ell+1)} \in \mathcal{B}^{(\ell+1)}(r_{\ell+1}),
\qquad
r_{\ell+1} \le \alpha_\ell r_\ell,
\]
with $\alpha_\ell < 1$. This completes the inductive step and proves the theorem.

\end{proof}
\subsection{Proof of Theorem \ref{thm:multi_basin}}
\label{app:multi_basin}
\begin{theorem}[Multi-basin partitioning]

Consider a task with $K$ common misconceptions. The hallucination subspace $\mathcal{H}^{(\ell)} = \{h^{(\ell)}(x) : x \in D_{\text{hall}}\}$ admits a Voronoi tessellation into $K$ basins centered at $\{\mu_1^{(\ell)}, \ldots, \mu_K^{(\ell)}\}$:

\[
\mathcal{B}_k^{(\ell)} = \left\{h \in \mathcal{H}^{(\ell)} : \|h - \mu_k^{(\ell)}\| \le \|h - \mu_j^{(\ell)}\| \forall j \neq k\right\}
\]

where each basin corresponds to a distinct misconception type with probability:

\[
P(\text{basin}_k | h) = \frac{\exp(-\|h - \mu_k\|^2 / 2\sigma^2)}{\sum_{j=1}^K \exp(-\|h - \mu_j\|^2 / 2\sigma^2)}.
\]    

\end{theorem}

\begin{proof}
    We introduce these assumptions for the proof:
    \begin{assumption}[Mixture structure of hallucination states]
\label{assump:mixture}
The hallucination subspace $\mathcal{H}^{(\ell)}$ is generated by a finite mixture
of $K$ latent misconception types $\{m_1,\dots,m_K\}$. Conditional on misconception
$m_k$, the hidden states are distributed as a Gaussian:
\[
h^{(\ell)} \mid m_k \;\sim\; \mathcal{N}\!\left(\mu_k^{(\ell)},\,\sigma^2 I\right),
\]
with equal prior probabilities $P(m_k)=1/K$.
\end{assumption}

\begin{assumption}[Distinct misconception centers]
\label{assump:distinct_centers}
The centers $\{\mu_k^{(\ell)}\}_{k=1}^K$ are distinct:
$\mu_k^{(\ell)} \neq \mu_j^{(\ell)}$ for $k \neq j$.
\end{assumption}

The theorem has two main steps to proof: (1) the geometric Voronoi partitioning, and (2) the probabilistic basin assignments.

To begin part one, let's take a look at the set of centers $\{\mu_1^{(\ell)},\dots,\mu_K^{(\ell)}\}$, define for each
$k$ the region
\[
\mathcal{B}_k^{(\ell)}
=
\left\{h \in \mathcal{H}^{(\ell)} :
\|h-\mu_k^{(\ell)}\| \le \|h-\mu_j^{(\ell)}\| \;\forall j \neq k
\right\}.
\]

By Assumption~\ref{assump:distinct_centers}, for any
$h \in \mathcal{H}^{(\ell)}$ the minimum of
$\{\|h-\mu_j^{(\ell)}\|\}_{j=1}^K$ is achieved by at least one index $k$.
Thus the collection $\{\mathcal{B}_k^{(\ell)}\}_{k=1}^K$ covers
$\mathcal{H}^{(\ell)}$.

Moreover, for $k \neq j$, the boundary between $\mathcal{B}_k^{(\ell)}$
and $\mathcal{B}_j^{(\ell)}$ is given by
\[
\|h-\mu_k^{(\ell)}\| = \|h-\mu_j^{(\ell)}\|,
\]
which defines a hyperplane orthogonal to $\mu_k^{(\ell)}-\mu_j^{(\ell)}$.
Hence the sets $\mathcal{B}_k^{(\ell)}$ form a Voronoi tessellation of
$\mathcal{H}^{(\ell)}$ induced by the centers
$\{\mu_k^{(\ell)}\}_{k=1}^K$.

That concludes the first part. The second part is the basin assignment task, where by Assumption~\ref{assump:mixture}, the likelihood of a hidden state
$h \in \mathcal{H}^{(\ell)}$ under misconception $m_k$ is
\[
p(h \mid m_k)
=
(2\pi\sigma^2)^{-d/2}
\exp\!\left(-\frac{\|h-\mu_k^{(\ell)}\|^2}{2\sigma^2}\right),
\]
where $d$ is the embedding dimension.

Using Bayes' rule and the uniform prior $P(m_k)=1/K$,
\begin{align*}
P(m_k \mid h)
&=
\frac{P(m_k)\,p(h \mid m_k)}
{\sum_{j=1}^K P(m_j)\,p(h \mid m_j)} \\[4pt]
&=
\frac{(1/K)\exp\!\left(-\|h-\mu_k^{(\ell)}\|^2/(2\sigma^2)\right)}
{\sum_{j=1}^K (1/K)\exp\!\left(-\|h-\mu_j^{(\ell)}\|^2/(2\sigma^2)\right)}.
\end{align*}

Canceling the common factor $(1/K)$ yields
\[
P(m_k \mid h)
=
\frac{\exp\!\left(-\|h-\mu_k^{(\ell)}\|^2/(2\sigma^2)\right)}
{\sum_{j=1}^K \exp\!\left(-\|h-\mu_j^{(\ell)}\|^2/(2\sigma^2)\right)}.
\]

Identifying misconception $m_k$ with basin $\mathcal{B}_k^{(\ell)}$
completes the proof.

\end{proof}

\subsubsection{Corresponding Multi-Basin Algorithm}

\begin{algorithm}[H]
\caption{Multi-Basin Detection for Misconceptions}
\label{alg:multi_basin}
\begin{algorithmic}[1]
\REQUIRE Hidden states $\{h_i\}$, labels $\{y_i\}$ (0=factual, 1=hall), candidate basins $K$
\STATE $H_{\text{fact}} \gets \{h_i : y_i = 0\}$, $H_{\text{hall}} \gets \{h_i : y_i = 1\}$
\STATE $\mu_{\text{ref}} \gets \frac{1}{|H_{\text{hall}}|}\sum_{h \in H_{\text{hall}}} h$

\STATE Compute total hallucination variance:
\[
\sigma_{\text{hall}}^2 \gets \frac{1}{|H_{\text{hall}}|}\sum_{h \in H_{\text{hall}}}\|h - \mu_{\text{ref}}\|^2
\]

\STATE $\{\mu_1, \ldots, \mu_K\} \gets \text{KMeans}(H_{\text{hall}}, K)$

\STATE Compute within-cluster variance:
\[
\sigma_{\text{within}}^2 \gets
\frac{1}{|H_{\text{hall}}|}
\sum_{k=1}^K \sum_{h \in \mathcal{B}_k} \|h - \mu_k\|^2
\]

\IF{$\sigma_{\text{within}}^2 / \sigma_{\text{hall}}^2 \ge \tau$}
    \STATE \textbf{return} single-basin collapse $(K=1)$
\ENDIF

\STATE Assign Voronoi labels:
\[
\hat{k}(h) = \arg\min_k \|h - \mu_k\|
\]

\STATE Train multi-class classifier on $(h_i, \hat{k}(h_i))$

\STATE \textbf{return} Basin centers $\{\mu_k\}$, classifier
\end{algorithmic}
\end{algorithm}
\section{Multi-Basin Partitioning}
In the next figures, we cluster hallucinated hidden states at the final output layer using a Gaussian mixture model to identify multiple hallucination basins. The results clusters are as visibly compact, and well-separeted, while corresponding to distinct misconception types. 
\newpage
\begin{figure}[t]
    \centering
    \begin{subfigure}{0.48\linewidth}
        \centering
      \includegraphics[width=\linewidth]{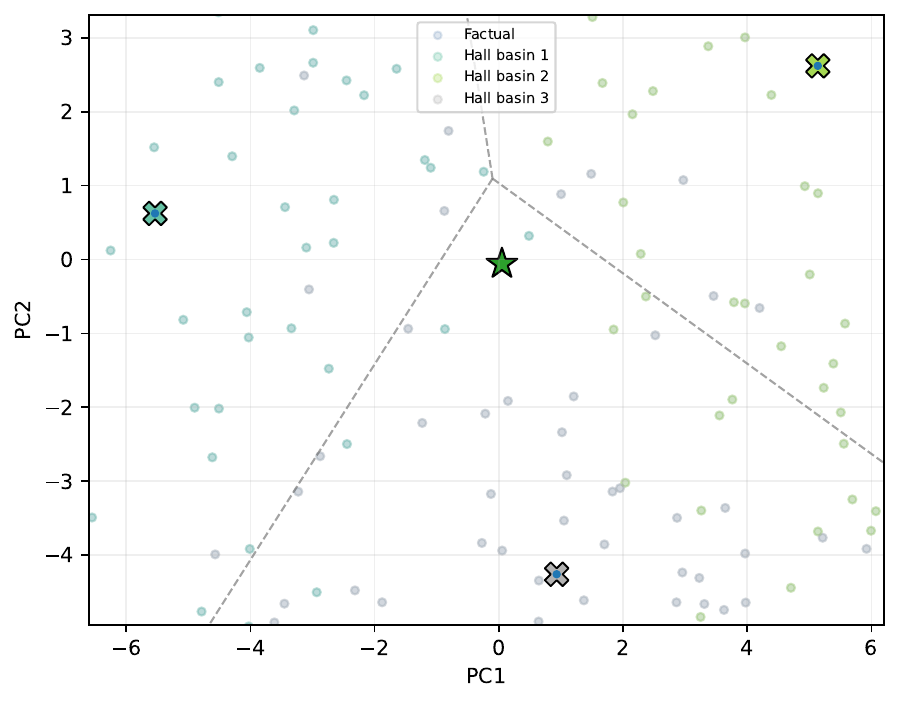}
        \caption{LLaMA-3.2 3B}
        \label{fig:multi_llama3b}
    \end{subfigure}
    \hfill
    \begin{subfigure}{0.48\linewidth}
        \centering
      \includegraphics[width=\linewidth]{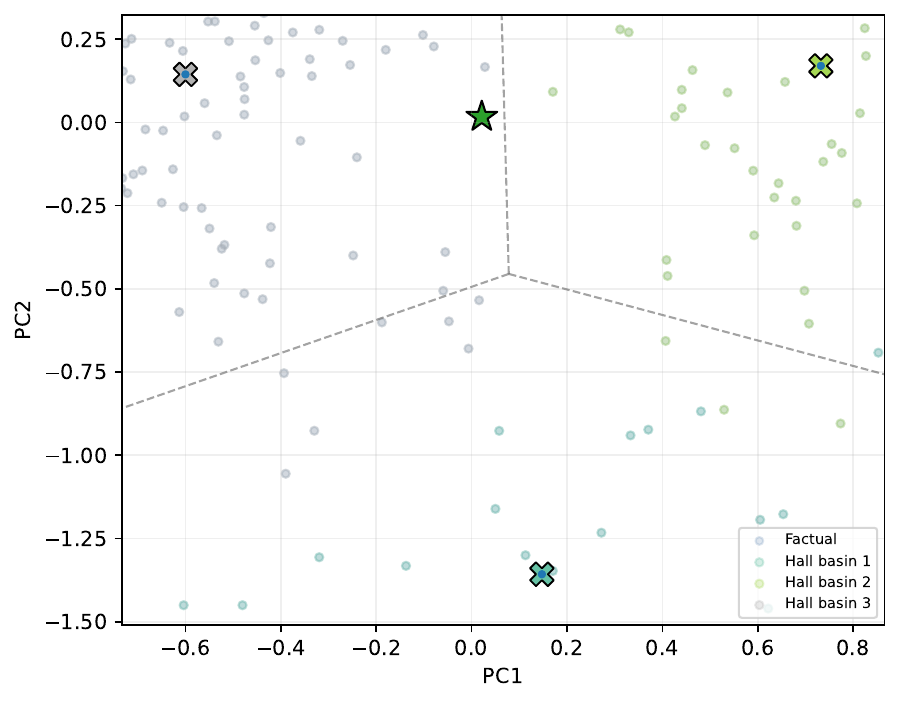}
        \caption{LLaMA-3.2 1B}
        \label{fig:multi_llama1b}
    \end{subfigure}

    \vspace{0.5em}

    \begin{subfigure}{0.48\linewidth}
        \centering
      \includegraphics[width=\linewidth]{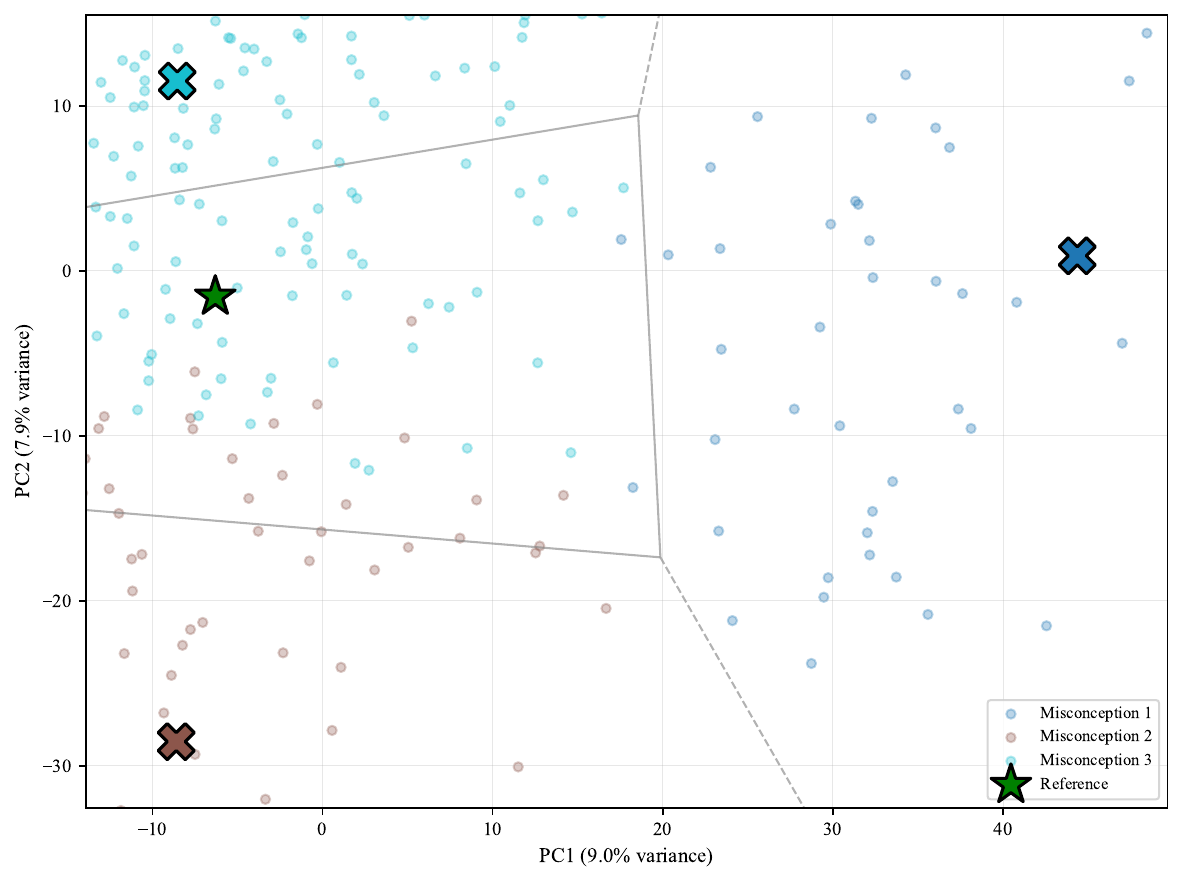}
        \caption{Qwen-2.5 1.5B}
        \label{fig:multi_qwen1p5b}
    \end{subfigure}
    \hfill
    \begin{subfigure}{0.48\linewidth}
        \centering
      \includegraphics[width=\linewidth]{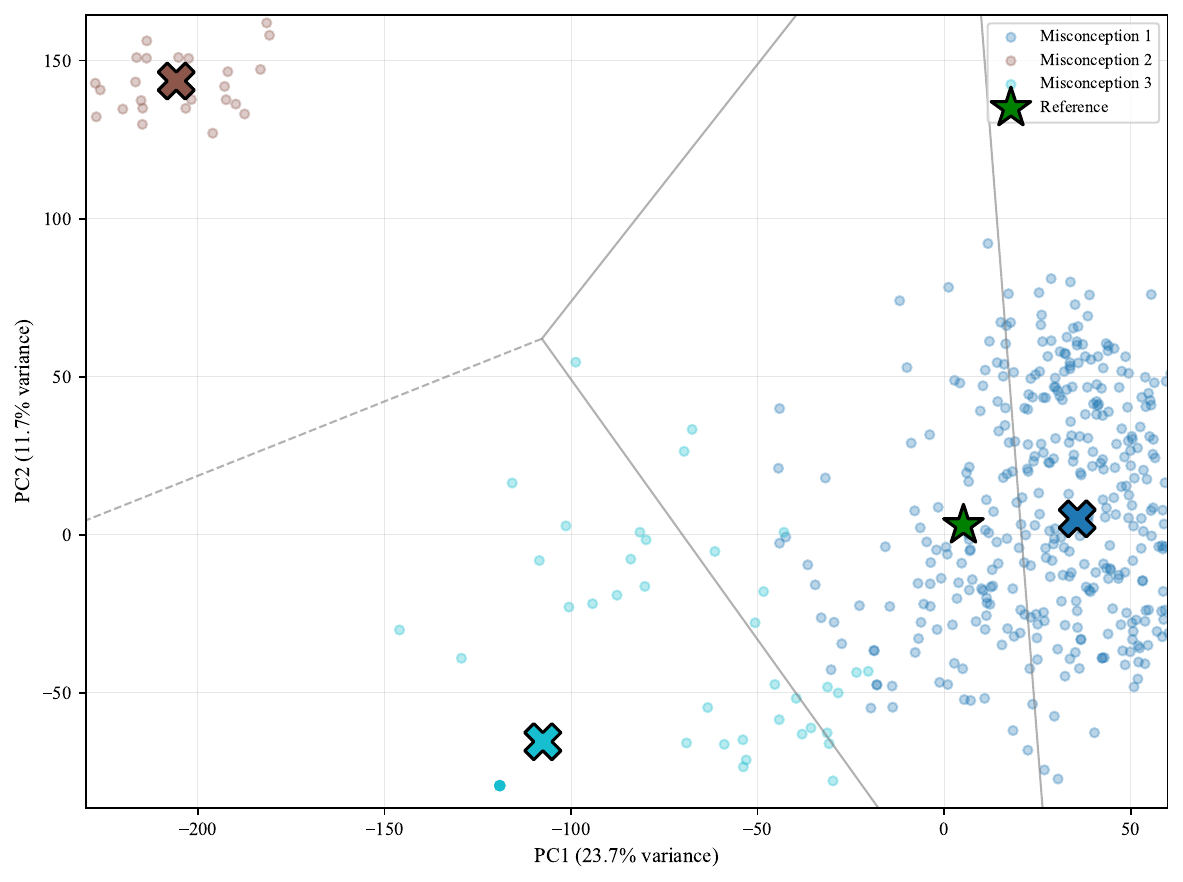}
        \caption{Gemma-2 2B}
        \label{fig:multi_gemma2b}
    \end{subfigure}

    \vspace{0.5em}

    \begin{subfigure}{0.48\linewidth}
        \centering
      \includegraphics[width=\linewidth]{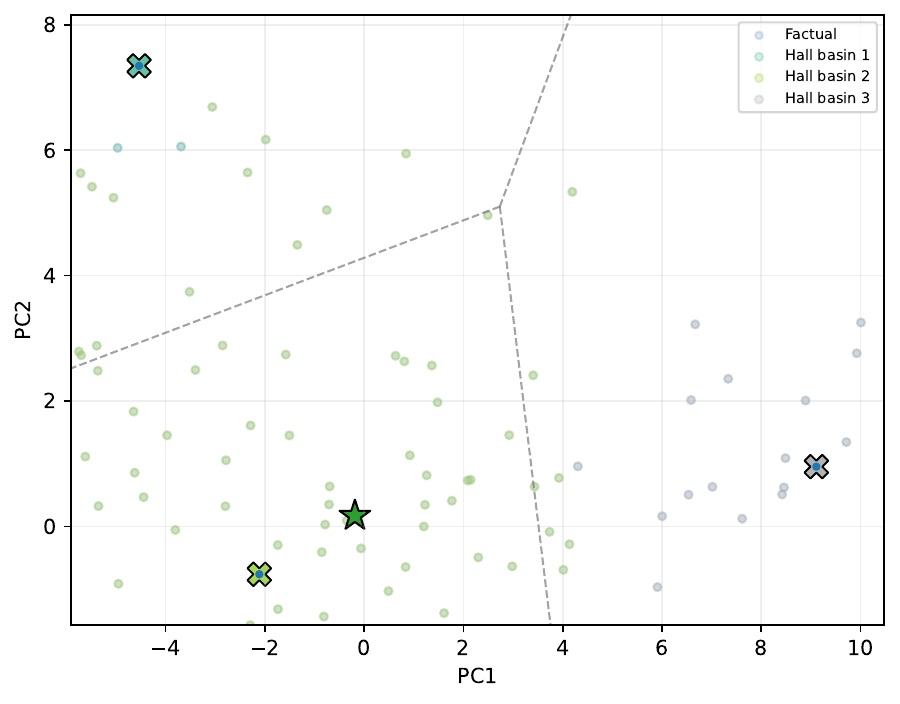}
        \caption{Llama-3.1 8B}
        \label{fig:multi_llama8b}
    \end{subfigure}
    \hfill
    \begin{subfigure}{0.48\linewidth}
        \centering
      \includegraphics[width=\linewidth]{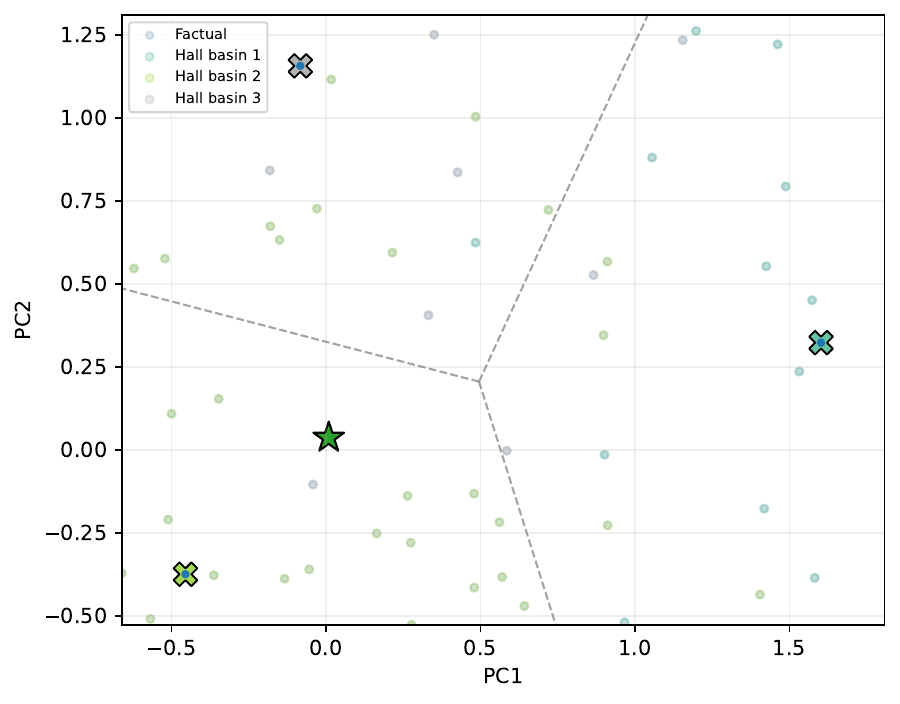}
        \caption{Mistalv0.3 7B}
        \label{fig:multi_mistral7b}
    \end{subfigure}
    \caption{Multi-basin Voronoi structure across models on TruthfulQA. Each panel shows distinct hallucination basins corresponding to different misconception modes.}
    \label{fig:multi_basin_all}
\end{figure}

\FloatBarrier
\subsection{Trajectory Stability}

We now formalize the link between the trajectory trapping phenomenon and loss of the dependence on the context.

\begin{definition}[Context Sensitivity]
    Let the sensitivity of the output distribution to latent perturbations be:
    \[
    \mathcal{S}(h) = \sup_{\|\delta||_2\le\epsilon} ||g(h+\delta) - g(h)||_1.
    \]
\end{definition}
\begin{theorem}[Stability Implies Context Insensitivity]
    Suppose: (1) the trajectory satisfies $h^{(\ell)}(x) \in \mathcal{B}_{\ell}(r)$, and (2) the readout function $g$ is $\kappa$-Lipschitz on $\mathcal{B}_{\ell}(r)$.
Then
$
\mathcal{S}(h^{(\ell)}(x)) \le \kappa\epsilon
$,
and in particular the output distribution is insensitive to latent perturbations that preserve membership in $\mathcal{B}_\ell (r)$.
\end{theorem}
\begin{proof}
    Follows from the Lipschitz assumption:
    $||g(h+\delta) - g(h)||_1 \le \kappa||\delta||_2 \le \kappa \epsilon$.
\end{proof}
Once a trajectory is trapped, variations within the hidden state no longer meaningfully affect the output. This is the mechanism by which hallucination-like behavior arises, which is in the fluent but context-insensitive generation.

\section{An Adaptive Steering Vector}
\label{sec:adaptive_steering}
\FloatBarrier

\subsection{Risk-Aware Steering}

Standard steering vectors apply a constant penalty $\lambda$ across all inputs, which often degrades performance on factual queries where no intervention is needed. To mitigate this, we propose a geometry-aware controller that dynamically scales $\lambda$ based on the hidden state's proximity to a hallucination basin.

We first define the static steering direction $v^{(\ell)}_{\text{steer}}$ as the difference between the factual and hallucinated centroids at layer $\ell$:
\[
v^{(\ell)}_{\text{steer}} = \mu^{(\ell)}_{\text{fact}} - \mu^{(\ell)}_{\text{hall}} = \frac{1}{|D_{\text{fact}}|} \sum_{x \in D_{\text{fact}}} h^{(\ell)}(x) - \frac{1}{|D_{\text{hall}}|} \sum_{x \in D_{\text{hall}}} h^{(\ell)}(x).
\]

To determine the intervention magnitude, we introduce two geometric features:

\begin{definition}[Local Contraction Ratio]
\label{def:local_kappa}
The rate at which the hidden state trajectory converges toward the basin center between layers $\ell$ and $\ell+1$:
\[
\kappa_{\text{local}}^{(\ell)}(h) = \frac{\|h^{(\ell+1)} - \mu^{(\ell+1)}\|_2}{\|h^{(\ell)} - \mu^{(\ell)}\|_2 + \epsilon},
\]
where $\epsilon$ is a small constant for numerical stability. A ratio $\kappa < 1$ indicates active collapse into the basin.
\end{definition}

Definition~\ref{def:distance} is also accompanied as the second geometric feature into the method. We aggregate these features into a risk signature vector $\Phi(x) \in \mathbb{R}^2$. The steering intensity is then determined by a learned scalar map $\lambda: \mathbb{R}^2 \to \mathbb{R}_+$ (e.g., a logistic regression trained to distinguish factual/hallucinated trajectories based on geometry):
\[
h^{(\ell)}_{\text{steered}}(x) = h^{(\ell)}(x) + \lambda(\Phi(x)) \cdot v^{(\ell)}_{\text{steer}}.
\]

\begin{algorithm}[H]
\caption{Geometry-Aware Adaptive Steering}
\label{alg:adaptive_steering}
\begin{algorithmic}[1]
\REQUIRE Model $f_\theta$, input $X$, centroids $\{\mu^{(\ell)}\}$, steering vectors $\{v^{(\ell)}_{\text{steer}}\}$, controller $\lambda(\cdot)$, layers $\mathcal{L}_{\text{steer}}$
\ENSURE Steered hidden states $\{h^{(\ell)}_{\text{steered}}\}$

\STATE $\{h^{(\ell)}(X)\}_{\ell=1}^{L} \gets f_\theta(X, \text{output\_hidden\_states}=\text{True})$

\FOR{$\ell \in \mathcal{L}_{\text{steer}}$}
    \STATE $d^{(\ell)} \gets \|h^{(\ell)}(X) - \mu^{(\ell)}\|_2$ 
    
    \IF{$\ell < \max(\mathcal{L}_{\text{steer}})$}
        \STATE $\kappa^{(\ell)} \gets \frac{\|h^{(\ell+1)} - \mu^{(\ell+1)}\|}{\|h^{(\ell)} - \mu^{(\ell)}\| + \epsilon}$ 
    \ELSE
        \STATE $\kappa^{(\ell)} \gets 1.0$ 
    \ENDIF
\ENDFOR

\STATE $\Phi(X) \gets [\min_{\ell}(d^{(\ell)}), \text{mean}_{\ell}(\kappa^{(\ell)})]$ 

\STATE $\lambda_X \gets \lambda(\Phi(X))$ 

\FOR{$\ell \in \mathcal{L}_{\text{steer}}$}
    \STATE $h^{(\ell)}_{\text{steered}} \gets h^{(\ell)}(X) + \lambda_X \cdot v^{(\ell)}_{\text{steer}}$
\ENDFOR

\STATE $y \gets \text{SteeredGeneration}(f_\theta, X, \{h^{(\ell)}_{\text{steered}}\})$
\end{algorithmic}
\end{algorithm}

\subsection{Empirical Validation of Algorithm~\ref{alg:adaptive_steering}}
We validate this empirically on Llama-3.2-1b, Llama-3.2-3b, Qwen-2.5-1.5b on HaluEval QA and Llama-3.2-1b on MuSiQue.

\begin{figure}[H]
\centering
\includegraphics[width=0.85\columnwidth]{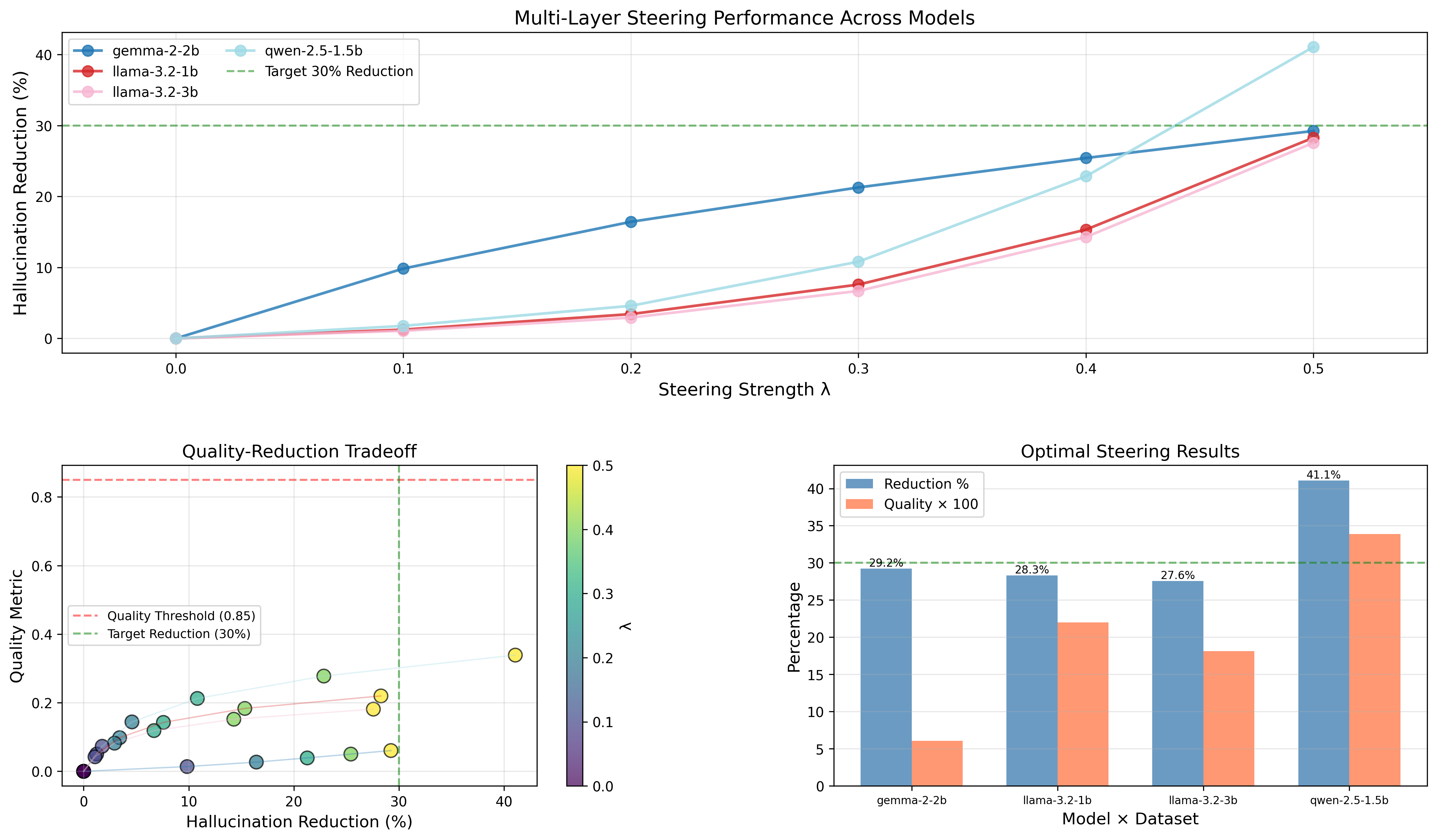}
\caption{Efficacy of Algorithm~\ref{alg:adaptive_steering} in hallucination reduction as a function of the steering strength $\lambda$.}
\label{fig:steering_comp}
\end{figure}

\section{Extended Empirical Validations}

\subsection{Autoregressive Irreversibility}

\begin{figure}[H]
\centering
\includegraphics[width=0.8\linewidth]{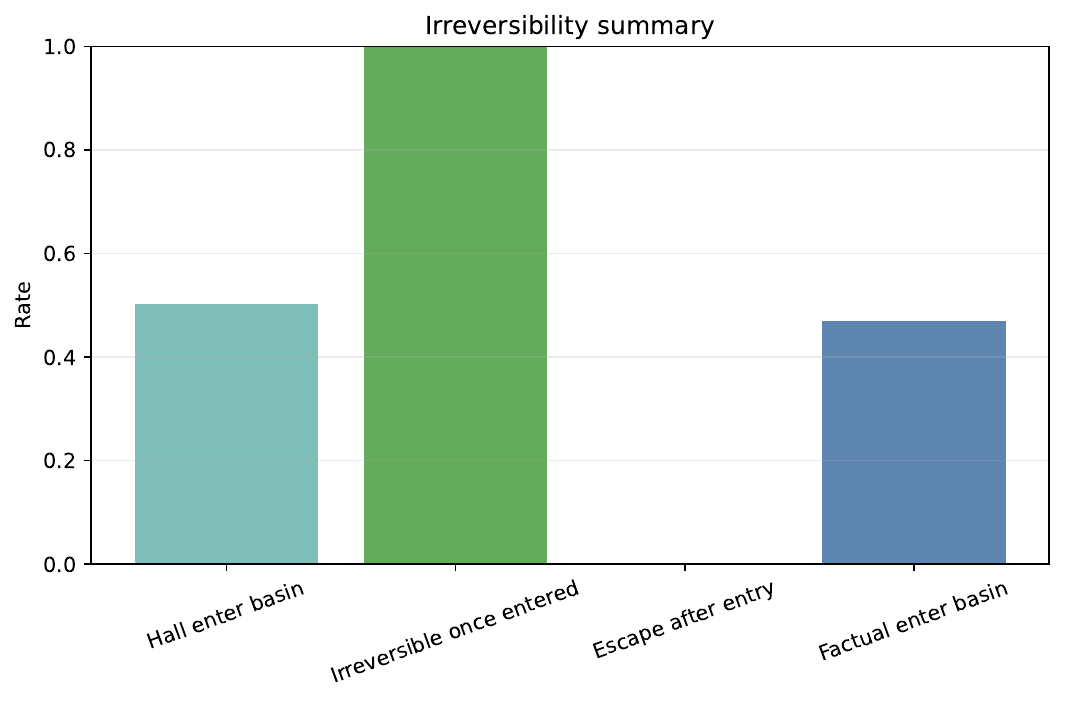}
\caption{Irreversibility summary under autoregressive decoding (HaluEval QA, Llama-3.2-1B, best layer). We report basin-entry, conditional irreversibility, escape-after-entry, and factual entry rates. This verifies Theorem~\ref{thm:traj_trap}.}
\label{fig:autoregressive_irreversibility_summary}
\end{figure}

\subsection{Layer-Wise Attention Entropy}

\begin{figure}[H]
\centering
\includegraphics[width=0.90\linewidth]{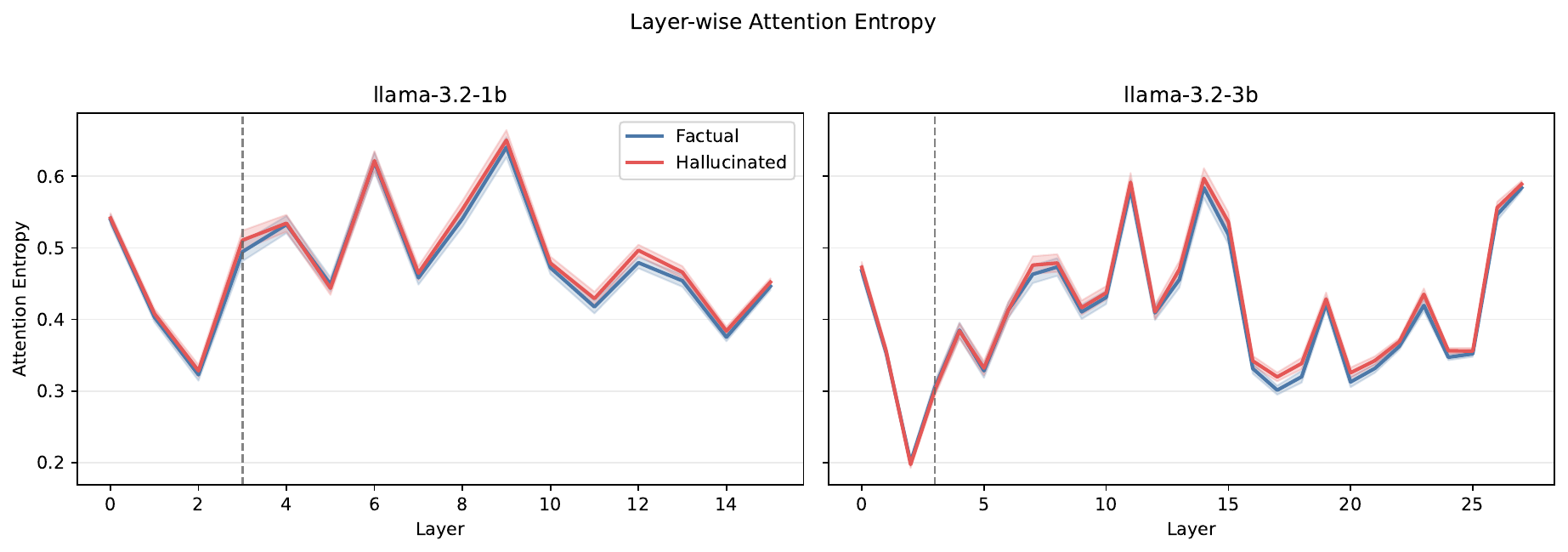}
\caption{Layer-wise attention entropy for factual versus hallucinated generations under uninformative contexts (autoregressive extraction). Entropy trends provide a complementary signal to basin-separation metrics. Supports the uniform attention assumption.}
\label{fig:attention_entropy_comparison}
\end{figure}

\newpage
\FloatBarrier
\subsection{Causality Intervention Paths}
\label{app:intervention}

This section presents figures of 3D PCA projections of middle-layer hidden activations with factual and hallucination samples plotted, together with the interpolation trajectory (Intervention Path) between their centroids. For each strength of steering $\alpha$ we overlay the in-model mean hidden states produced via injection of the learned basin direction during the generation forward passes. Together, the geometry and the in-model interventions result in direct causal evidence that a basin direction in hidden-state space both organizes the hallucination examples and, when injected during generation, it drives the model into higher hallucination probabilities.

\OneWideFigure{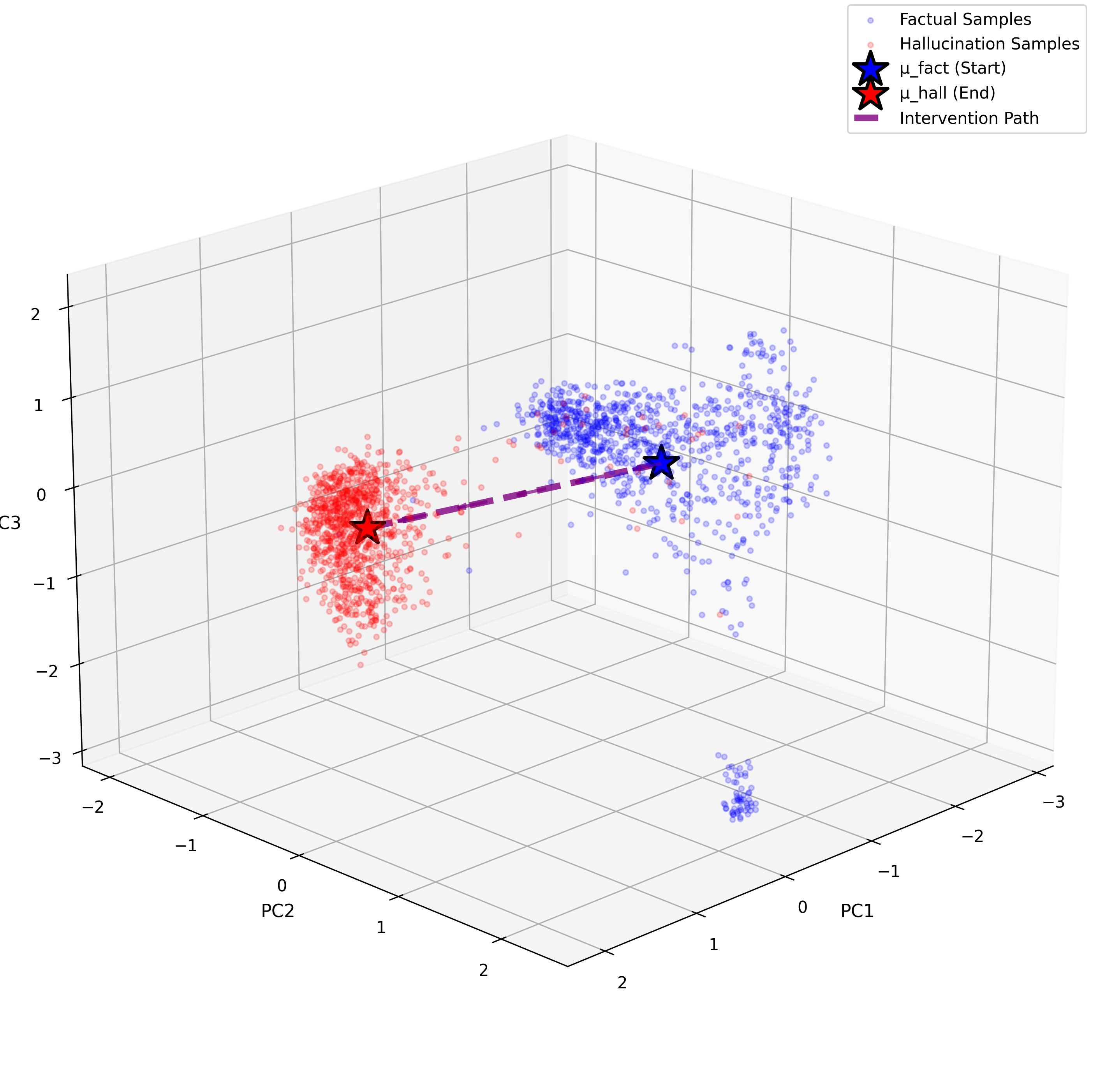}{ Causal Intervention Paths: Llama-3.2-1B (HaluEval QA)}{fig:causality_traj_llama1b}
\OneWideFigure{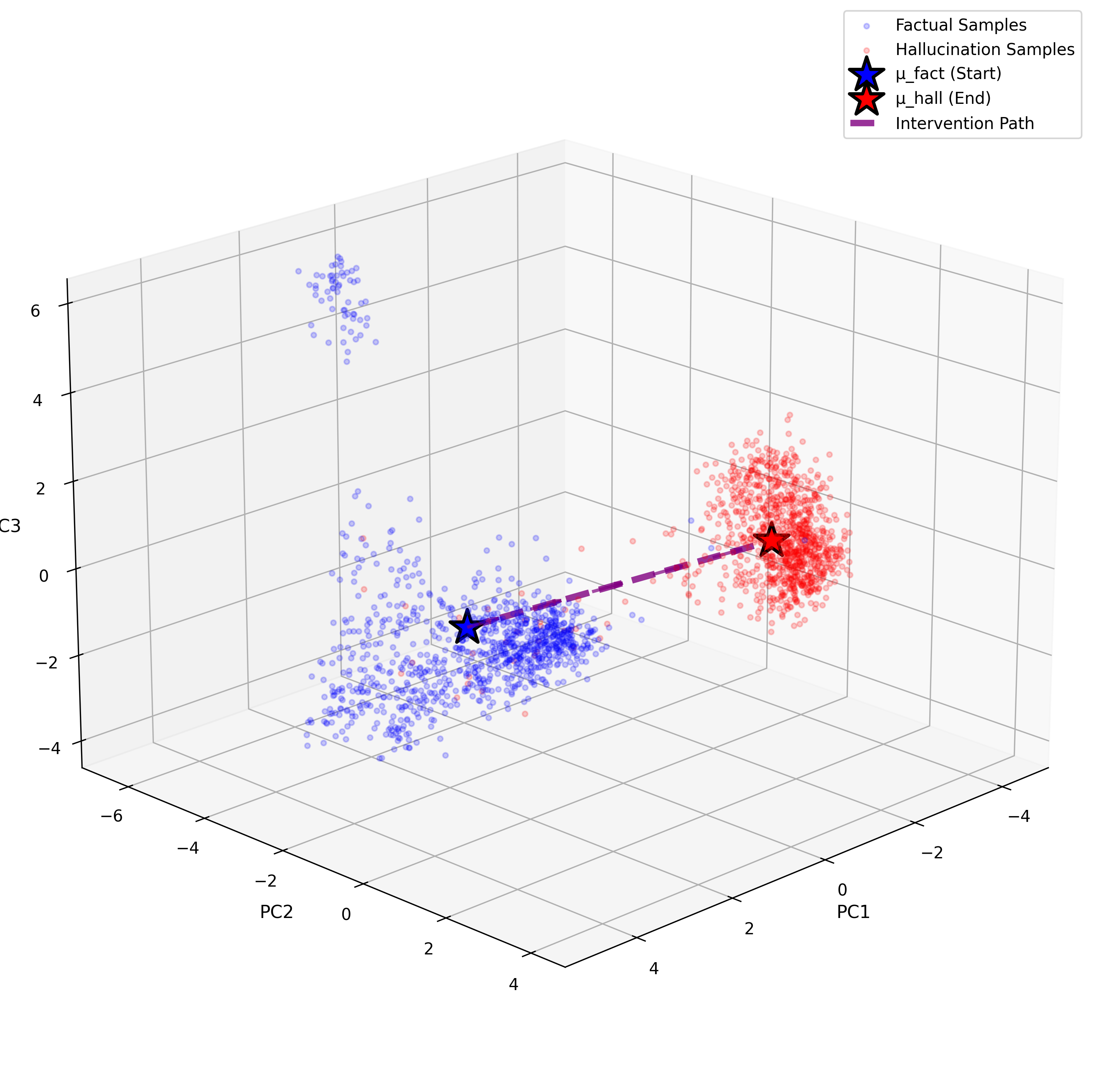}{ Causal Intervention Paths: Llama-3.2-3B (HaluEval QA)}{fig:causality_traj_llama3b}
\OneWideFigure{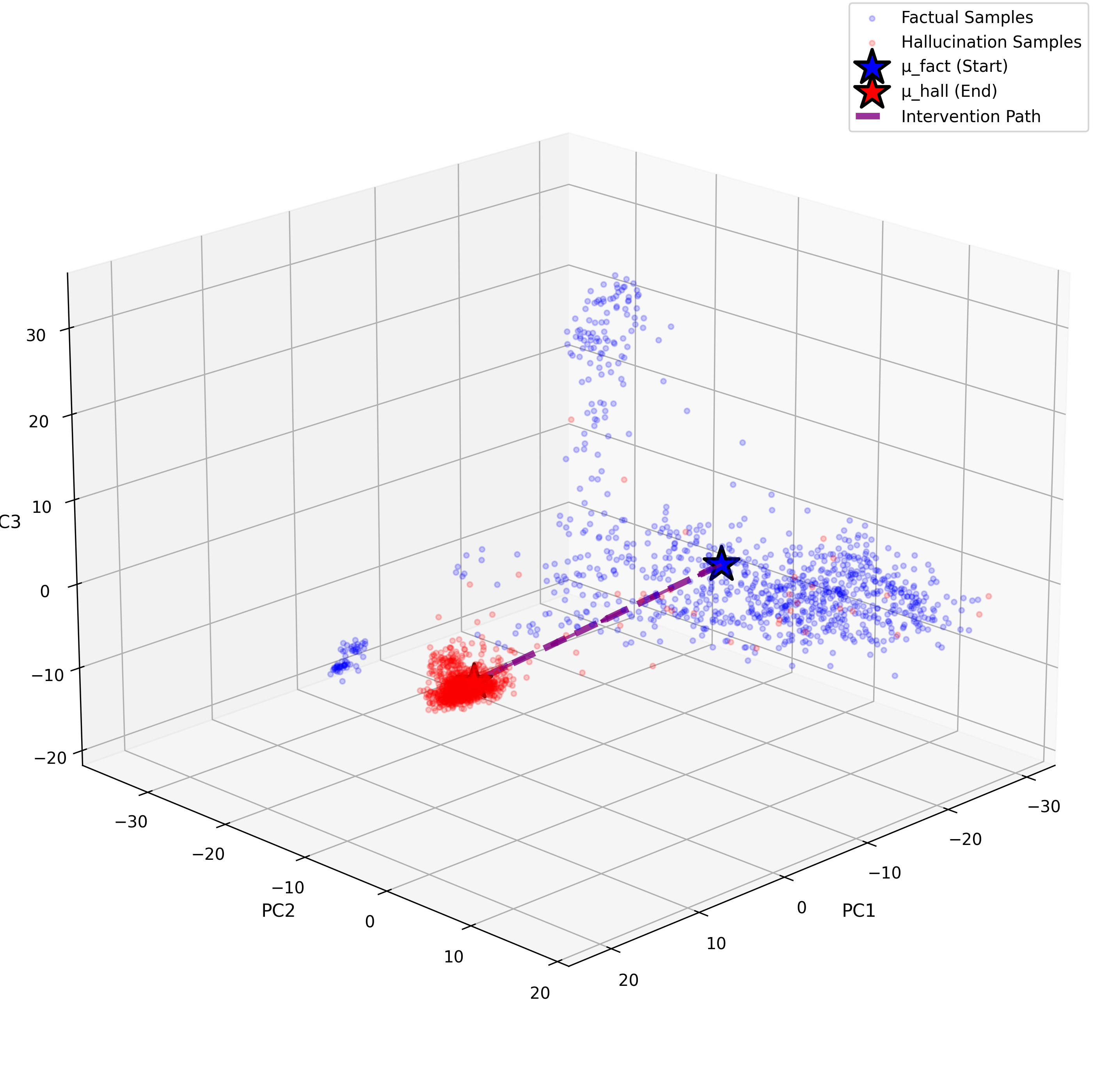}{ Causal Intervention Paths: Qwen-2.5-1.5B (HaluEval QA)}{fig:causality_traj_qwen1p5b}

\newpage
\FloatBarrier
\subsection{2D Layer Evolutions}
This result provides samples at every 3rd layer of the model. And it evaluates and projects the subsequent 2D PCA plot. 
\OneWideFigure{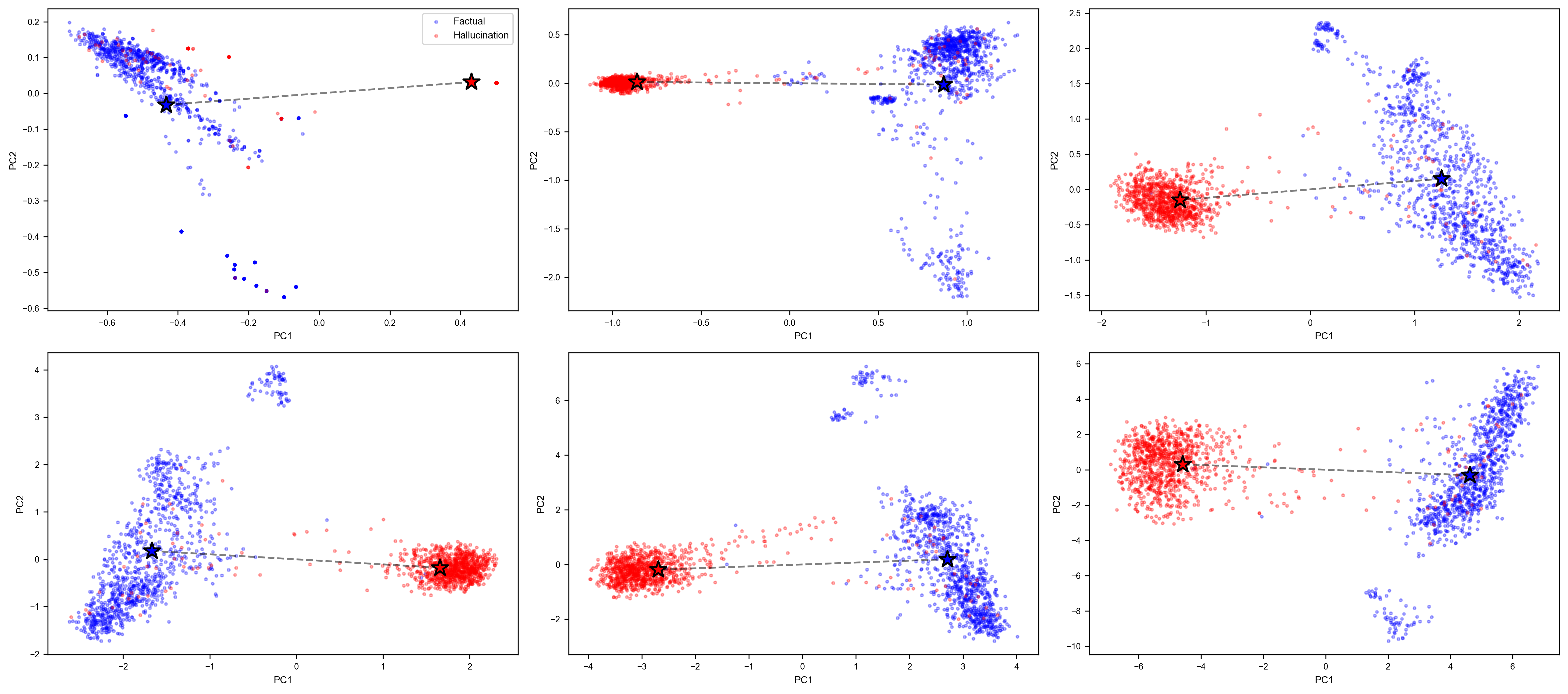}{2D PCA Evolution: Llama-3.2 1B (QA).}{fig:layer2d_llama1b_qa}
\OneWideFigure{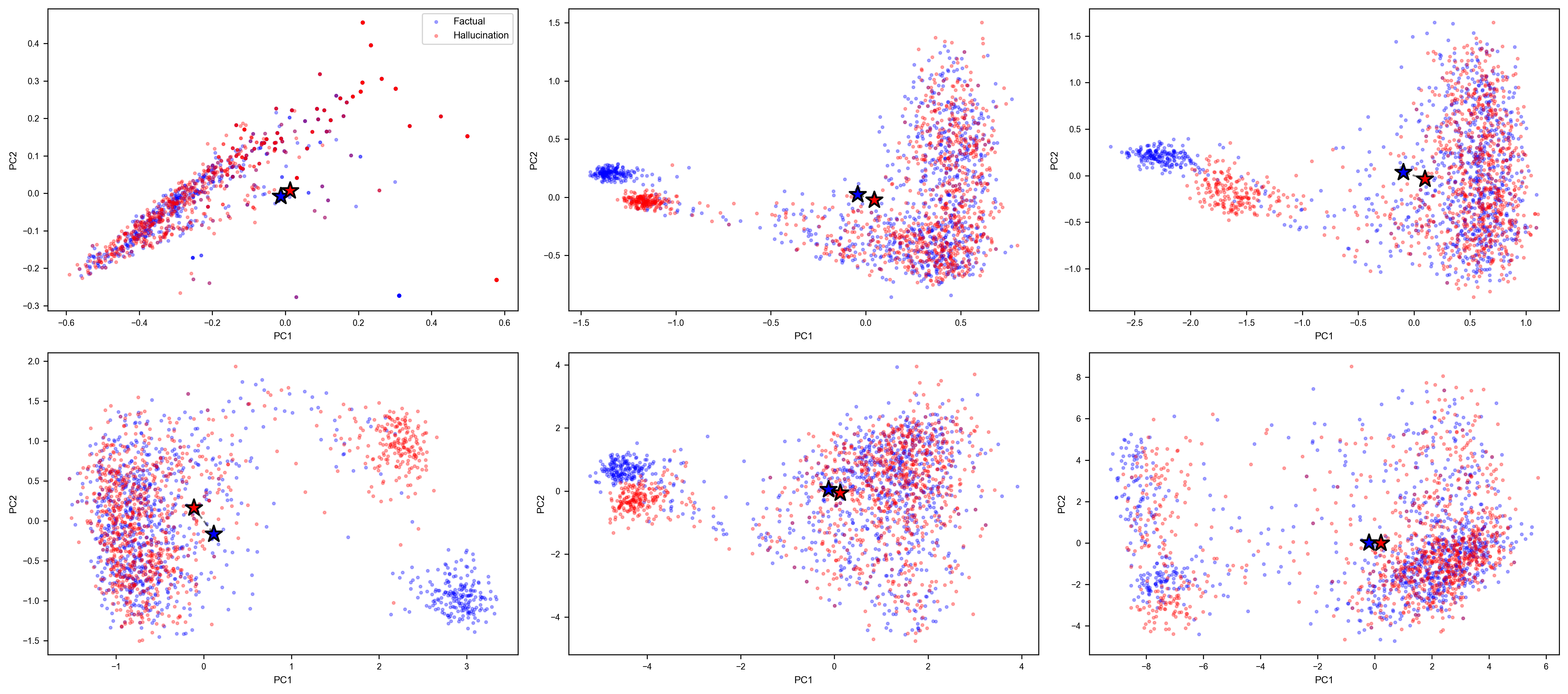}{2D PCA Evolution: Llama-3.2 1B (Summarization)}{fig:layer2d_llama1b_sum}
\OneWideFigure{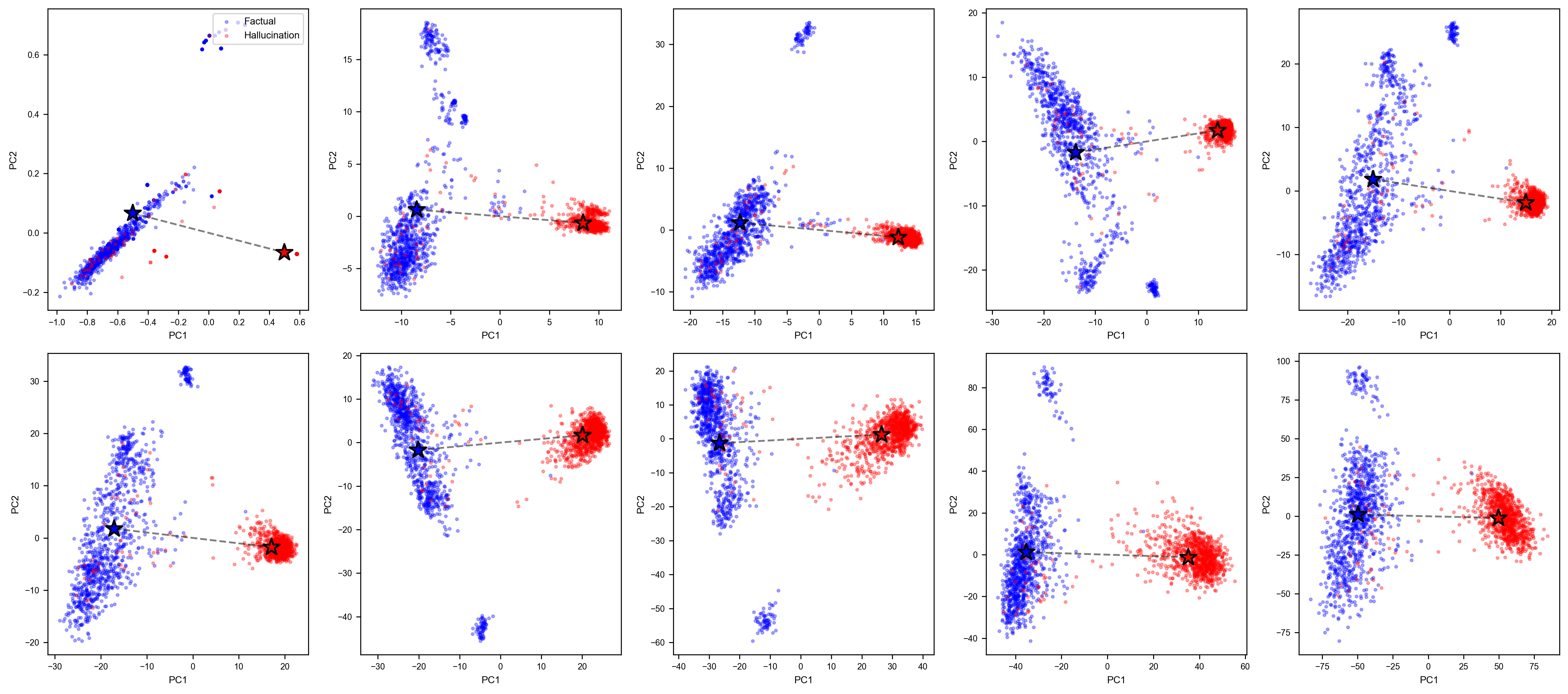}{2D PCA Evolution: Qwen-2.5 1.5B (QA)}{fig:layer2d_qwen1p5b}
\OneWideFigure{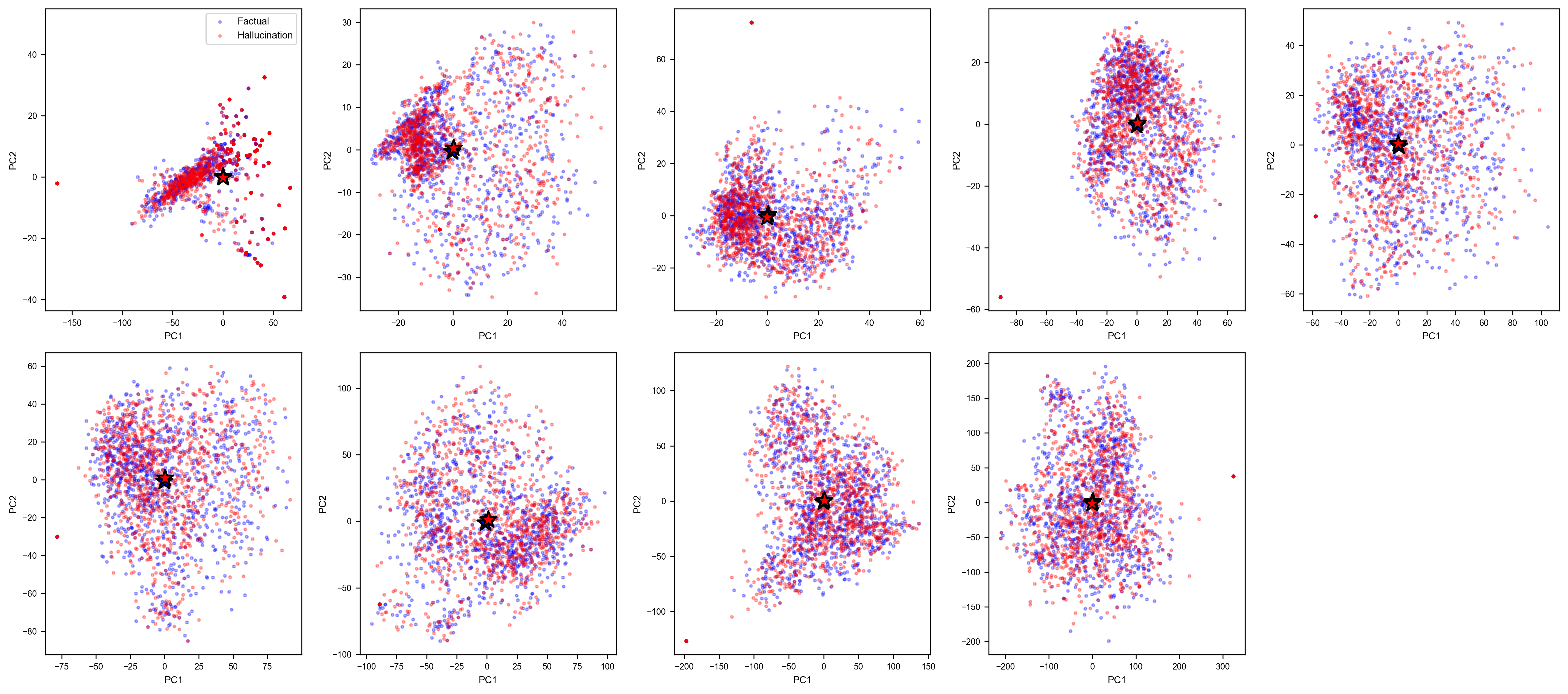}{2D PCA Evolution: Gemma-2 2B (Summarization)}{fig:layer2d_gemma2b}

\newpage
\subsection{3D Layer Evolutions}
This result provides samples at every 3rd layer of the model. And it evaluates and projects the subsequent 3D PCA plot. 
\noindent\includegraphics[width=\linewidth,height=0.72\textheight,keepaspectratio]{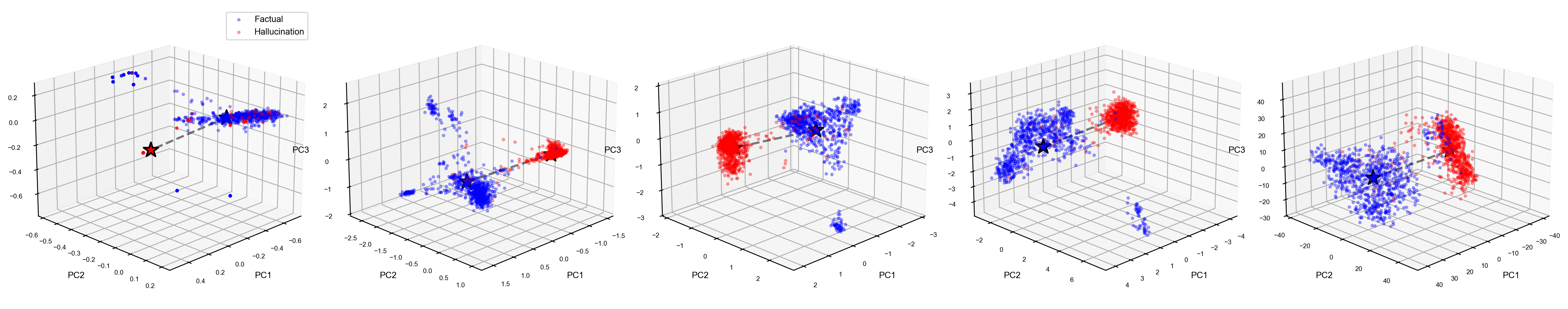}
\captionof{figure}{3D PCA Evolution: Llama-3.2 1B (QA)}
\label{fig:layer3d_llama1b_qa}

\noindent\includegraphics[width=\linewidth,height=0.72\textheight,keepaspectratio]{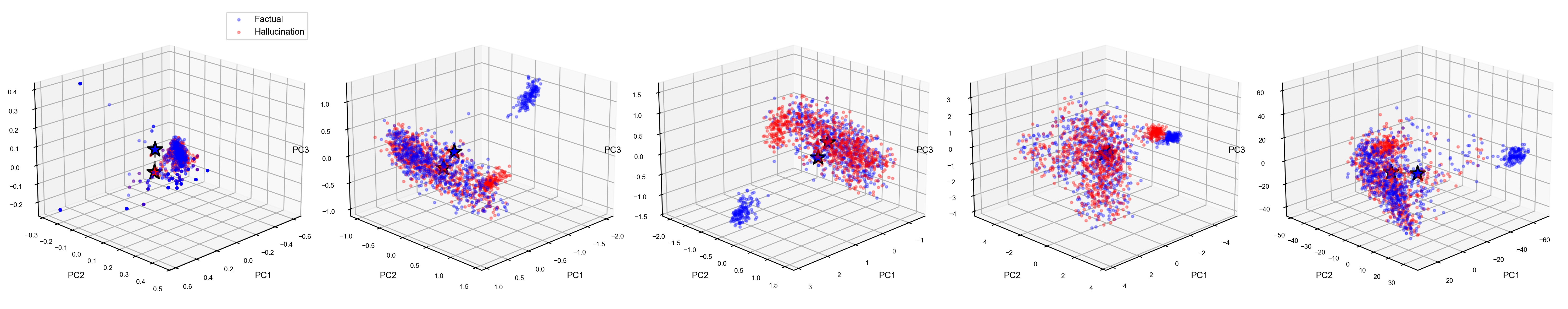}
\captionof{figure}{3D PCA Evolution: Llama-3.2 1B (Summarization)}
\label{fig:layer3d_llama1b_sum}

\noindent\includegraphics[width=\linewidth,height=0.72\textheight,keepaspectratio]{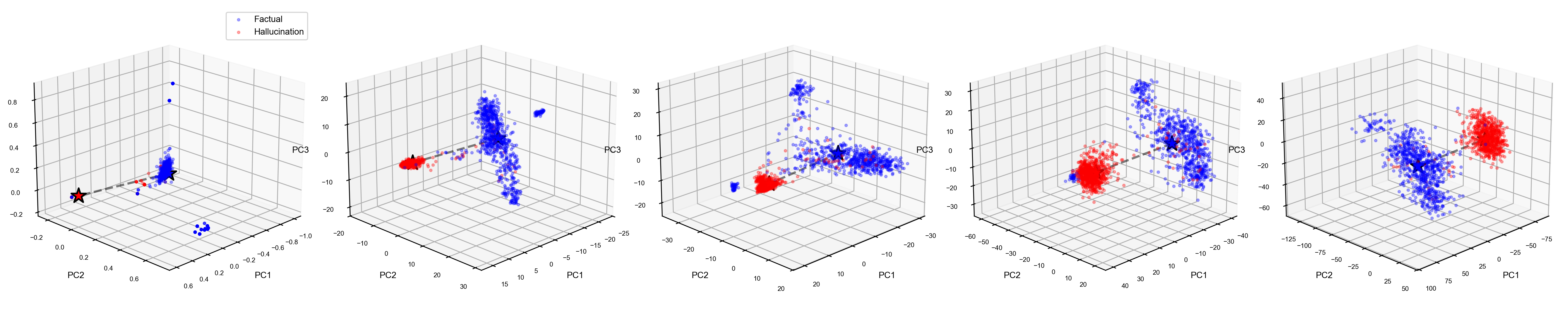}
\captionof{figure}{3D PCA Evolution: Qwen-2.5 1.5B (QA)}
\label{fig:layer3d_qwen1p5b}

\noindent\includegraphics[width=\linewidth,height=0.72\textheight,keepaspectratio]{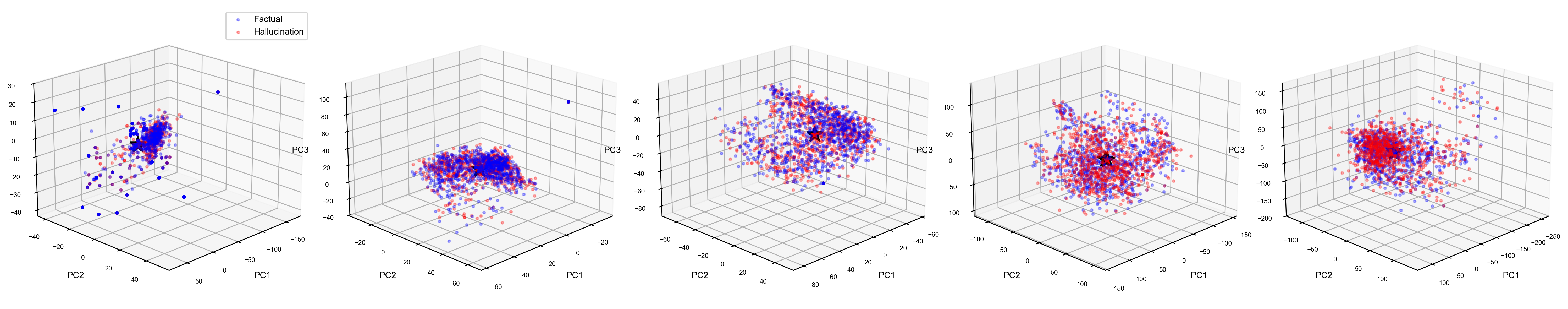}
\captionof{figure}{3D PCA Evolution: Gemma-2 2B (Summarization)}
\label{fig:layer3d_gemma2b}

\newpage
\section{Limitations and Future Work}
\label{limitations}
\paragraph{Access and Observability} Our approach assumes access to internal hidden states and the ability to estimate reference centroids from uninformative contexts, which may limit direct application to closed or inference-only APIs. Exploring black-box approximations is the natural direction for future work.

\paragraph{Computational Limitations} All experiments in this work were conducted using a single NVIDIA RTX 4060 Laptop GPU with 8GB of VRAM. While this setup is sufficient for controlled analysis of representation dynamics in mid- to large-sized open-source language models, it constrained the scale of models, context lengths, and experimental variants that could be evaluated. In particular, these limits made systematic experimentation with substantially larger models, dense hyperparameter sweeps, and long-horizon autoregressive decoding computationally impractical.

\paragraph{Models and Tasks} Experiments focus and run with several mid-sized open-source models and standard benchmarks for hallucination, which is sufficient to demonstrate and validate the phenomenon but doesn't necessarily guarantee transfer to different model architectures, that may be proprietary or multimodal. Thus, a natural direction for future work would include explorations into more model families, longer contexts, and domain-specific tasks.

\paragraph{Theoretical Approximations} Our theoretical results have a majority of simplifying assumptions (e.g., approximate attention uniformity), which are intended to clarify the mechanism rather than try to characterize all transformer variants present in the literature. We view this as potentially being more generalizable to families of transformer variants than to explain every single one.

\paragraph{Hidden-State Extraction} We analyze hidden-state trajectories under autoregressive decoding, but some runs retain low effective sample counts after filtering invalid trajectories. Improving throughput and expanding balanced autoregressive splits remain important future work.

\section{Hyperparameters, Reproducibility and Code Availability}

\subsection{Causal Intervention Protocol}
\label{app:discussion}
\paragraph{What we actually do?} During generation we inject a steering vector into the model’s hidden activations at a chosen transformer layer so that all downstream computation (attention, layernorm/FFN nonlinearities, and future token predictions) observes the perturbation. This is a true in-model causal intervention.

\subsection{Evaluation Protocol}
\label{app:evaluation_protocol}
For all experiments presented in this paper, we enforced strict train/test separation for both the labeling and reference construction processes. Each dataset is randomly split into 70\% training and 30\% test (the data is stratified wherever applicable). All layerwise reference statistics and measures (this includes factual and hallucinated centroids $\mu^{(\ell)}$, covariance estimates $\Sigma^{(\ell)}$, and basin radii $r^{(\ell)}$ are all \textbf{\textit{computed exclusively on the training split}} and then frozen. Test examples are never used at any stage of the reference estimation process, threshold selections, or for tuning hyperparameters. Hallucination labels are strictly derived from the dataset-provided annotations (e.g., ground-truth answers). These labels \textbf{do not} depend on internal model signals or detector-based outputs. Detection performance is evaluated only the held-out test split. All reported AUROC values are averaged over the course of three independent random splits; we report the mean, and for each split, we estimate $95\%$ confidence interval using $20$ bootstrap resamples of the test set, and report the mean alongside the $95\%$ confidence interval. 

\subsection{Prompt Templates}
\label{app:prompts}

\textbf{HaluEval QA Prompt:}
\begin{verbatim}
Question: {question}
Answer: {answer}
\end{verbatim}

\textbf{MuSiQue Prompt:}
\begin{verbatim}
Context: {paragraphs}
Question: {question}
Answer: {answer}
\end{verbatim}

\textbf{FEVER Prompt:}
\begin{verbatim}
Claim: {claim}
This claim is [factual/false].
\end{verbatim}

\textbf{TruthfulQA GPT-4 Judge Prompt:}
\begin{verbatim}
Question: {question}
Correct answers: {correct_answers}
Incorrect answers: {incorrect_answers}
Model response: {generated_text}

Is the model response factually correct or a hallucination?
Output only: FACTUAL or HALLUCINATION
\end{verbatim}

\begin{table}[H]
\begin{center}
\small
\caption{Global data generation, model loading, and evaluation hyperparameters shared across all experiments.}
\label{tab:global_hparams}
\begin{tabular}{@{}ll@{}}
\toprule
\textbf{Parameter} & \textbf{Setting} \\
\midrule
Random seed & 42 (NumPy, PyTorch, scikit--learn) \\
Determinism & Seeds fixed; some GPU ops may remain non-deterministic \\
Batch size (train / extraction) & 8 (training loaders); 4 (hidden-state extraction) \\
Generation batch size & Per experiment; demo runs use batch $=$ \#prompts (typically 3) \\
Maximum sequence length & 512 tokens (truncation applied) \\
Tokenizer padding & \texttt{pad\_token} $=$ \texttt{eos\_token}; left padding for decoder-only models \\
Numerical precision & 4-bit quantization when supported; float16 fallback otherwise \\
Compute device & CUDA when available; CPU fallback \\
Train / test split & 70\% / 30\%, stratified by label \\
Split seed & 42 (deterministic split) \\
Hidden-state extraction & Last-token hidden representation (unless stated otherwise) \\
Centroid definition & Class-wise mean of hidden vectors (Euclidean centroid) \\
Feature normalization & \texttt{StandardScaler} fitted on training split \\
Detection classifier & Logistic regression (L2, \texttt{lbfgs}, \texttt{max\_iter}=1000, seed=42) \\
Distance-based score & Ratio $d_{\mathrm{fact}} / (d_{\mathrm{hall}} + \varepsilon)$ \\
Covariance estimation & Ledoit--Wolf shrinkage (when robust covariance is required) \\
PCA (visualization) & PCA with $n{=}3$, no whitening, seed=42 \\
Figure export & Raw tensors saved as compressed NPZ \\
\bottomrule
\end{tabular}
\end{center}
\vspace{2pt}
\footnotesize\textit{Note:} Unless stated otherwise, all hallucination probabilities $\mathbb{P}(\mathrm{hall})$ are obtained from classifier-predicted probabilities over last-token hidden states. In-model causal interventions are gated via an environment flag and evaluated using stochastic generation on five prompts (no teacher forcing).
\end{table}

\begin{table}[H]
\begin{center}
\small
\caption{Experiment-specific hyperparameters used across detection, causality, steering, and ablation studies.}
\label{tab:experiment_hparams}
\begin{tabular}{@{}lll@{}}
\toprule
\textbf{Category} & \textbf{Parameter} & \textbf{Setting} \\
\midrule
Detection
& Evaluation layer & Middle layer ($\lfloor L/2 \rfloor$) \\
& Covariance estimate & Ledoit--Wolf shrinkage \\
& Detection model & Logistic regression (L2) on standardized features \\
\midrule
Causality
& Interpolation grid & $\alpha \in \{0, 0.1, \ldots, 1.0\}$ \\
& Control directions & Random + orthogonalized (Gram--Schmidt) \\
& Effect metric & Fold change $\frac{\mathbb{P}(\mathrm{hall})_{\text{int}}}{\mathbb{P}(\mathrm{hall})_{\text{base}}}$ \\
\midrule
Steering
& Intervention layers & $\lfloor L/3 \rfloor$, $\lfloor 2L/3 \rfloor$ \\
& Steering vector & $v_{\text{basin}} = \mu_{\text{hall}} - \mu_{\text{fact}}$ \\
& Strength grid & $\lambda \in \{0, 0.1, \ldots, 0.5\}$ \\
\midrule
Ablation
& Layer sweep & Sliding or exhaustive windows over layers \\
& Reported metrics & AUROC change, fold reduction \\
\midrule
Visualization
& Projection & PCA ($n{=}3$) on up to $10^3$ samples per class \\
\bottomrule
\end{tabular}
\end{center}
\vspace{2pt}
\footnotesize\textit{Note:} In-model interventions mirror offline interpolations by applying identical $\alpha$-scaled hidden-state shifts during generation. All hallucination probabilities are computed from the same detection classifier.
\end{table}

\end{document}